\documentclass[twoside]{article}
\usepackage[accepted]{aistats2024}
\usepackage[utf8]{inputenc}

\usepackage{hyperref}
\usepackage{cite}
\usepackage{microtype}
\usepackage{caption, subcaption, graphicx}
\usepackage{siunitx}
\usepackage{multirow}

\usepackage{amssymb, amsthm, dsfont}

\newtheorem{theorem}{Theorem}
\newtheorem{lemma}{Lemma}
\newtheorem{proposition}{Proposition}
\theoremstyle{remark}
\newtheorem*{example}{Example}
\newcommand{\R}{\mathbb{R}}
\newcommand{\N}{\mathbb{N}}
\renewcommand{\P}{\mathbb{P}}
\newcommand{\E}{\mathbb{E}}
\newcommand{\F}{\mathcal{F}}
\newcommand{\J}{\mathcal{J}}
\renewcommand{\H}{\mathcal{H}}
\renewcommand{\O}{\varnothing}
\DeclareMathOperator{\dom}{dom}
\DeclareMathOperator{\kl}{kl}
\DeclareMathOperator{\KL}{KL}
\DeclareMathOperator{\Alt}{Alt}
\DeclareMathOperator{\add}{add}
\DeclareMathOperator{\rmv}{rm}

\begin{document}
\runningauthor{élise crepon, Aurélien Garivier, Wouter M. Koolen}
\twocolumn[
  \aistatstitle{
    Sequential Learning of the Pareto Front for Multi-objective Bandits
  }
  \aistatsauthor{élise crepon\footnotemark
    \And Aurélien Garivier \And Wouter M. Koolen}
  \aistatsaddress{ENS Lyon \And ENS Lyon \And CWI and Twente University}
]
\begin{abstract}
  We study the problem of sequential learning of the Pareto front in
  multi-objective multi-armed bandits. An agent is faced with $K$ possible arms
  to pull. At each turn she picks one, and receives a vector-valued reward. When
  she thinks she has enough information to identify the Pareto front of the
  different arm means, she stops the game and gives an answer. We are interested
  in designing algorithms such that the answer given is correct with probability
  at least $1-\delta$. Our main contribution is an efficient implementation of an
  algorithm achieving the optimal sample complexity when the risk $\delta$ is small.
  With $K$ arms in $d$ dimensions $p$ of which are in the Pareto set, the
  algorithm runs in time $O(Kp^d)$ per round.
\end{abstract}

\section{INTRODUCTION}
Stochastic multi-armed bandits have emerged as a fundamental framework for
studying sequential learning. In the classic setting of scalar rewards, the UCB
algorithm solves the regret minimization problem and the Track-and-Stop
algorithm solves the best arm identification problem. In this paper we are
interested in the extension to vector-valued rewards, which is the arena for
multi-criterion optimization. Here the problem of identifying the best arm
generalizes to identifying the subset of arms with Pareto optimal means
\cite{auer2016pareto}. We study this problem in the fixed confidence setting.
That is, the learner is given a confidence parameter. She sequentially collects
noisy vector-valued rewards from a finite-armed bandit. After having collected
enough data, the learner stops and outputs a subset of arms. The goal of the
learner is to identify with high probability the correct Pareto set. Our
approach is based on instantiating the Track-and-Stop framework
\cite{garivier2016optimal} to the Pareto front identification problem so as to
obtain an algorithm with optimal sample complexity \footnotetext{preferred
capitalization}\footnote{This can be either asymptotic optimality as in the
original \cite{garivier2016optimal}, or some later finite confidence refinement
\cite{purex.games}.}. The Track-and-Stop framework has recently seen tremendous
success across pure exploration problems in bandits and RL. Examples include
best arm identification in Spectral \cite{kocak2021epsilon}, Stratified
\cite{subpopulations}, Lipschitz \cite{purex.games}, Linear
\cite{degenne2020gamification,jedra2020optimal}, Dueling
\cite{haddenhorst2021identification}, Contextual
\cite{tirinzoni2020asymptotically,hao2020adaptive} and Markov bandits
\cite{moulos2019optimal}. Other objectives include Top-$m$ identification
\cite{chen2017nearly,chen2017adaptive}, MaxGap identification
\cite{katariya2019maxgap}, Thresholding \cite{garivier2017thresholding}, Monte
Carlo Tree Search \cite{maximinarm}, optimal policy identification in MDPs
\cite{al2021adaptive}, and minimizing Tail-Risk \cite{tailrisk}. The framework
has also been instantiated to Pareto front identification
\cite{kone2023adaptive}. The Track-and-Stop template is in some sense universal:
the starting point is an information-theoretic, instance-dependent sample
complexity lower bound of min-max form (see our Proposition~\ref{prop:lbd}
below). The learning algorithm is designed to match this lower bound by solving
(an estimate of) that min-max problem. For that, in turn, it suffices to
calculate a certain gradient \cite{purex.games}. Yet here the details become
problem-specific, in the sense that the tractability of this gradient
computation varies across problems. Some identification objectives have
closed-form solutions, some have efficient special-purpose optimizers, some can
leverage a reduction to generic convex optimization, and for others nothing much
is known. An overarching methodology remains elusive, and as such it is
important to extend our toolbox by solving particular hard cases, of which
Pareto front identification is a prime example. Our contribution is, at its
core, an efficient implementation \footnote{The source code used to run all
experiments included in the paper is available at
\url{https://github.com/elise-crepon/sequential-pareto-learning-experiments}.}
of the gradient computation required for executing Track-and-Stop for Pareto
front identification. With that problem-specific computational kernel
implemented, the general scheme instantiates and we obtain an instance-optimal
fixed confidence learner. Our combinatorial and algorithmic innovations reduce
the run-time in $p$ from the naive exponential $O(K d^{p+1})$ to polynomial $O(K
p^d)$ per round. This is a reasonable computation cost for instances with a
large number of arms and a small number of objectives.

We conclude here by mentioning extensions and alternative versions.
\cite{kone2023adaptive} study approximate Pareto front identification. Skyline
identification is a special case of Pareto front identification \cite{8437618}.
Multi-objective optimization is also studied from a regret minimization
perspective. Achieving a vector-valued expected reward not too far from the
Pareto front (in some distance metric) is studied, both in stochastic and
adversarial bandits \cite{pmlr-v70-busa-fekete17a,pmlr-v202-xu23i}.
\cite{zuluaga2013active} motivates the relevance of studying multi-ojective
optimization.

\subsection{Setting}
We use the setting of multi-objective multi-armed bandits, which is the
following: given $K$ independent probability distributions on $\R^d$, $\nu =
(\nu_k)_{k \in [K]} \in V$, with respective means $\mu = (\mu_k^j)_{k \in [K], j \in
[d]}$, at every time step $t \in \N$ an agent gets to pick an arm $A_t \in [K]$
and receives an independent reward $X_t \in \R^d$ drawn from $\nu_{A_t}$. The
objective here is not to maximize cumulative reward over time but to identify as
fast as possible (under a correctness constraint) the best performing arms.
However, since we are in a multi-objective setting, we have no way of
identifying a single best performing arm as we could do in a single objective
framework: an arm might perform really well on one objective $j \in [d]$ but get
poor results on another one, or an arm could rank averagely but on all the
objectives. We have no way of discriminating one against the other.

For that reason, we are interested in identifying all the Pareto optimal arms.
To be Pareto optimal, an arm must not be dominated by another one which means
having its performance on all of the objectives be poorer than a single other
arm. An arm which would have another one dominating it, is one about which we
are sure that it has a better counter part, whereas an arm which as no one
dominating it is optimal since we have no way of comparing how different
objective compare with each other. We are then interested in identifying all of
the Pareto optimal arms as fast as possible. Letting $\F_t = \sigma(X_1, \cdots,
X_t)$ be the sigma-field generated by the observations up to time $t$. A
strategy is then defined by \emph{a sampling rule} $(A_t)_t$ where $A_t \in [K]$
is $\F_{t-1}$ measurable, \emph{a stopping rule} $\tau$, which is a stopping time
with respect to $(\F_t)_t$, and \emph{an answer rule} $P_\tau \subseteq [K]$ that
is $\F_\tau$-measurable, which is the set of arm indices the learner assumes to be
the Pareto set.

Given a risk $\delta > 0$, we call a strategy $\delta$-PAC if it ensures that the answers
it gives at the end of its runs are correct with a confidence $\delta$, i.e. $\P(P_\tau
\text{ is wrong}) \leq  \delta$. This family of problem is called \emph{pure exploration}
and has already been well studied, in particular in the case where a single
answer is correct which is our case. We apply the results from the literature to
our specific problem. Also notice that while our setting is mainly focused on
multi-objective, it includes the single-objective framework within the special
case $d = 1$. This well-studied case of best-arm identification will serve as a
reference throughout the paper, some of the difficulties that we present having
their (degenerated) equivalent in dimension 1.

\subsection{Pareto optimality}
To formalize the definition of the Pareto set, we introduce the following binary
relation. An arm with distribution $\nu$ on $\R^d$ and mean $\mu \in \R^d$ is said
to be \emph{dominated} by an arm $\nu'$ with mean $\mu'$, which we denote by $\nu
\preceq \nu'$ (or equivalently $\mu \preceq \mu'$) iff $\forall j \in [d], \mu_j \leq 
\mu'_j$. This means that $\nu'$ performs better than $\nu$ on all the $d$ different
objectives. For a specific $\mu$, we create a partial order between the indices of
the arms given by $k_0 \preceq_\mu k_1 \iff \mu_{k_0} \preceq \mu_{k_1}$. This
comparison is a partial order but is not connected, hence within a finite set it
may have multiple maxima. We call these maxima the (strict) Pareto set and we
denote it by $p(\mu) \triangleq \max_{\preceq_\mu} [K] = \arg\max_{\preceq, k \in
[K]} \mu_k$. The Pareto set is defined as the indices of the points rather than
the points directly because while our leaner has access to the indices, it
doesn't have access to the points directly.

In this paper, we give an algorithm based on Track-and-Stop to identify the
Pareto set of multi-variate Gaussians. We provide a careful analysis of its
complexity. We also tackle the special case of dimension two and give an
improved complexity in this case. In the first section of this paper, we give a
formal definition of Pareto optimality. In Section 2, we recall the
Track-and-Stop framework and motivate why it applies to our problem. In Section
3, we detail our algorithm and analyze its complexity.

\section{LOWER BOUND ON THE SAMPLE COMPLEXITY AND ALGORITHM}
This section uses the results of \cite{garivier2016optimal} to derive a lower
bound on the sample complexity and to find an efficient algorithm to solve our
problem.

We use the formalism of active sequential hypotheses learning. We let $M$ be a
set of $K$ arms bandit models and $\H = \{\H_i \,|\, i \in [n]\}$ a finite set
of disjoint hypotheses (which is a partition of $M$). We introduce $h: M \to \H$
which is the function that associates to any $\nu \in M$ the only $\H_i$ that
contains $\nu$. We are interested in algorithms that can identify the hypothesis
that contains $\nu$. The sampling, stopping and answer rules defined in the
introduction still holds only the final answer given is $H_\tau$ the hypothesis the
learner assumes the models they are interacting with is in. In this context, a
$\delta$-PAC strategy is any strategy that can ensure that $\P_\nu (\nu \in H_\tau) (= \P_\nu
(h(\nu) = H_\tau)) \leq  \delta$ for all $\nu \in M$.

We denote by $\Alt(\nu) \triangleq \{ \nu' \in M \,|\, h(\nu') \neq h(\nu) \} = ¬ h(\nu)$
the subset of our model space $M$ which contains all the models that have a
different answer than the one of $\nu$. This set of models is important because
for a player to make a mistake they need to confuse the model they get
samples from with a model in $\Alt(\nu)$. Hence for any algorithm to stop, it
needs to get enough information to distinguish the current model from all models
in $\Alt(\nu)$ with risk at most $\delta$.

\cite{garivier2016optimal} introduce the following lower bound for the number of
samples needed for active sequential hypotheses testing with a unique valid
hypothesis in the bandit framework.
\begin{proposition}[Sample complexity lower bound]\label{prop:lbd}
  Given a set of models $M$, a finite set of disjoint hypotheses $\H =
  (\H_i)_{i \in [n]}$ which is a partition of $M$ and a risk parameter $\delta>0$,
  any $\delta$-PAC strategy is such that for every $\nu \in M$:
  $$ \E_\nu (\tau_\delta) \geq  \kl(\delta, 1-\delta) \cdot T^*(\nu) \;,$$
  where
  $$ T^*(\nu)^{-1} \triangleq \sup_{w \in \triangle_K} \inf_{\lambda \in \Alt(\nu)}
    \sum_{k \in [K]} w_k \KL(\nu_k, \lambda_k) \;. $$
\end{proposition}

Our task of Pareto set identification is an instance of this problem: our
hypotheses are for each subset of arms $[K]$ the set of models for which this
set is the Pareto set. For a given model, the only correct hypothesis is the one
associated with its Pareto set and the models in $\Alt(\nu)$ are such that their
Pareto set is not the same as that of $\nu$.
\begin{equation}\label{eq:altpar}
  \Alt(\nu) \triangleq \{ \nu' \in V\,|\, p(\nu') \neq p(\nu) \} \;.
\end{equation}

As noted in the same paper, this lower bound hints us toward an efficient
sampling rule. If we were to know $\nu$, then the maximizer $w^*$ of the
optimization problem $T^*(\nu)$ gives us the fastest sampling rule that is able to
make the difference between $\nu$ and the models in $\Alt(\nu)$. However, we don't
know $\nu$ upfront. The Track-and-Stop algorithm proposes to solve the
optimization problem with estimates of the model and to correct for possible
bias with some forced exploration. The algorithm also comes with a stopping and
recommendation rule that we import from the literature.

However, using this algorithm requires us to be able to solve the optimization
problem behind $T^*(\nu)$. For best arm identification, \cite{garivier2016optimal}
propose a clever yet special-purpose algorithm. However, that approach does not
work for Pareto set identification making the problem much harder. Since, for
$$ D_w (\nu,\lambda) \triangleq \sum_{k \in [K]} w_k \KL(\nu_k, \lambda_k) \;,$$ the
function $w \in \triangle_K \mapsto \inf_{\lambda \in \Alt(\nu)} D_w (\nu, \lambda)$ is concave
with respect to $w$, we can learn it using gradient ascent. Moreover, as we are
refining our estimates at each time step, we can do an online gradient ascent
and do only one step of the gradient ascent per time step. In order to perform
gradient ascent, we need to be able to solve and find the minimizer of
\begin{equation}\label{eq:arginf}
  \begin{array}{cc}
    \min & D_w(\nu, \lambda) \;. \\
    \text{w.r.t.} & \lambda \in \Alt(\nu)
  \end{array}
\end{equation}

However, the computation of the minimal transportation cost from $\nu$ to a $\lambda \in
\Alt(\nu)$ that changes the Pareto set is not a convex function and requires a
specific solving procedure. Because it carries more geometric intuition, we
tackle here the case of Gaussian random variables with identity covariance
(i.e.\ the objectives are independent from one another). Hence, our models are
fully parametrized by their means $\mu$, which we use as a stand-in for $\nu$ when
talking about them. Our main contribution is to propose an efficient algorithm
to solve this optimization problem in the case of Gaussian random variables. We
provide a general analysis for higher dimensions and refine it for the case of
dimension two. Under these assumptions, the $w$-weighted transportation cost
$D_w$ between two models equals:
\begin{equation} \label{eq:Dwgauss}
  D_w (\mu, \lambda) \triangleq \sum_{k \in [K]} \frac{w_k}{2} ||\mu_k - \lambda_k||^2\;.
\end{equation}

Though we will focus only on solving this specific optimization problem for the
rest of the paper without delving in the inner workings of the Track-and-Stop
algorithm, we give here a few details on how we instantiate it. Once we obtain a
gradient in $w$ from solving \eqref{eq:arginf} at a specific $w$, we do a single
step of the Hedge algorithm using the gradient-norm-adaptive tuning. Both
taking a single optimization step and this specific gradient ascent algorithm
are detailed in the literature. See
\cite{purex.games,degenne2020gamification,wang2021fast} for references.

\section{MINIMIZING THE TRANSPORTATION COST}
We give  in this section a general procedure to compute the minimal
transportation cost between $\mu$ and a $\lambda$ that changes the Pareto front. We
show that the computation of this cost can be split in two sub-procedures that
are independent from each other. We analyze the complexity of each of these
algorithms.

\begin{theorem}[Algorithmic complexity of the minimal transportation cost]
  \label{thm:cplx}
  The minimal transportation cost \eqref{eq:arginf} to change the Pareto front
  of our multi-variate Gaussians model \eqref{eq:Dwgauss} and its minimizer can
  be computed in
  $$ O\left(\big(K (p+d) + d^3 p \big)\binom{p + d - 1}{d-1} \right)\;, $$
  where $p$ is the number of Pareto optimal points in our model.
\end{theorem}

We introduce the following lemma, with proof in Appendix~\ref{app:pf.l1}, to
help us find the solution to \eqref{eq:arginf} by splitting $\Alt$ as defined in
\eqref{eq:altpar} in subpieces on which the optimization will be more easily
done.
\begin{lemma}[Splitting the domain]\label{lem:splitting}
  Let $\mu,\lambda \in \R^{K×d}$
  $$ p(\mu) \neq p(\lambda) \iff
    \begin{aligned}[t]
      & \exists \{k_0,k_1\} \subseteq p(\mu): k_0 \preceq_\lambda k_1 \\
      & \vee\exists k_0 \notin p(\mu) \forall k \in p(\mu) : k_0 \npreceq_\lambda k\,.
    \end{aligned} $$
\end{lemma}

In words, for $\lambda$ to have a different Pareto set than $\mu$ it is necessary and
sufficient that either a point from the Pareto set of $\mu$ is dominated in $\lambda$ by
another point from the Pareto set of $\mu$, or that a point that is on the Pareto
set of $\mu$ is no longer dominated in $\lambda$ by any of the points from the Pareto
set of $\mu$. Splitting the problem this way allows us to design efficient
procedures to find the minimum transportation cost from $\mu$ to a $\lambda$
that changes its Pareto set.

Given $\mu \in \R^{K×d}$, we define
\begin{align*}
  \Alt^{\rmv}&: \{k_0,k_1\} \subseteq p(\mu)
    \mapsto \{ \lambda \in \R^{K×d} \,|\, k_0 \preceq_\lambda k_1 \}  \hbox{ and}\\
  \Alt^{\add}&: k_0 \notin p(\mu)
    \mapsto \{\lambda \in \R^{K×d} _,|\, \forall k \in p(\mu)\, k_0 \npreceq_\lambda k \}\;.
\end{align*}
Lemma~\ref{lem:splitting} allows us to say that $\Alt^{\add} (k_0)$ for all $k_0
\notin p(\mu)$ and $\Alt^{\rmv} (k_0, k_1)$ for all $\{k_0, k_1\} \subseteq p(\mu)$
provide a covering of $\Alt(\mu)$. We can thus solve the minimization
independently for each of them and then take the minimal value of these as our
minimal transportation cost:
\begin{equation} \begin{aligned}\label{eq:splitting}
  \Alt(\mu)=
    &\left(\bigcup_{\{k_0, k_1\} \subseteq p(\mu)} \Alt^{\rmv}(k_0, k_1)\right)\\
    &\cup \left(\bigcup_{k_0 \notin p(\mu)} \Alt^{\add}(k_0)\right)\;.
\end{aligned} \end{equation}

We will now refer to the first case as removing a point from the Pareto set and
to the second one as adding a point on the Pareto set, but we want to emphasize
that the first case won't necessary yield the smallest cost to remove the given
point from the Pareto set and the second one will not necessarily add the
focused point to the Pareto set.

\subsection{Removing a point from the Pareto set}
In this section we prove the following lemma:
\begin{lemma}[Cost of removing a point from the Pareto set]\label{lem:removing}
  Given $\{k_0,k_1\} \subseteq p(\mu)$, the minimal transportation cost for
  \begin{equation}
    \begin{array}{cc}
      \min & D_w(\mu, \lambda)\\
      \text{w.r.t.} & \lambda \in \Alt^{\rmv}(k_0, k_1)
    \end{array}
  \end{equation} is
  $$ \frac{1}{2} \frac{w_{k_0} w_{k_1}}{w_{k_0} + w_{k_1}}
    \sum_{j \in [d]} \left(\max \{0,\mu^j_{k_0} - \mu^j_{k_1}\}\right)^2. $$
  This cost and the associated minimizer can be computed in $O(d)$ and then the
  algorithmic complexity for all $\Alt^{\rmv}$ is $O(p^2d)$.
\end{lemma}

\begin{proof}
Let $\{k_0, k_1\} \subseteq p(\mu)$, we are interested in computing the smallest
transportation cost from $\mu$ to $\lambda$ such that in $\lambda$ the point $k_1$ now
dominates $k_0$. Moving any other point than $k_0$ and $k_1$ in $\lambda$ doesn't
affect whether $k_1$ dominates $k_0$, hence this is superfluous and would only
cost us more, so we can restrict our analysis to $\lambda$ that only moves $k_0$ and
$k_1$. Now, let $\J$ be the set of dimensions alongside which $\mu^j_{k_0} \geq 
\mu^j_{k_1}$. Since our transportation cost is separable alongside each dimension,
then moving our points alongside any other axis that the ones in $\J$ would not
help create the domination and would bear some extraneous cost. As such we
can restrict ourselves to $\lambda$ that only move $k_0$ and $k_1$ alongside $\J$.
Using again that the transportation cost separability, we can split our analysis
along the different axis independently. Now, the cost of inverting $\mu^j_{k_0}$
and $\mu^j_{k_1}$ for $j \in \J$ is a known problem from best arm identification.
We compute the value here but it is possible to see \cite{garivier2016optimal}
for a reference. So our optimization problem boils down to
$$ \inf_{x_\lambda \leq  y_\lambda} w_x \frac{1}{2} (x_\mu-x_\lambda)^2
  + w_y \frac{1}{2} (y_\mu - y_\lambda)^2 \;.$$

The inf will be reached at a point where $x_\lambda = y_\lambda \triangleq s$. Taking the
derivative
in $s$ of the cost function, we get $w_x (s - x_\mu) + w_y (s - y_\mu)$ which is
null at $s^* \triangleq \frac{w_x x_\mu + w_y + y_\mu}{w_x + w_y}$ yielding the
following minimum transportation cost $\frac{1}{2} \frac{w_x w_y}{w_x+w_y} (x_\mu
- y_\mu)^2$. Now summing alongside the axis of $\J$ yields
$$ \frac{1}{2} \frac{w_{k_0} w_{k_1}}{w_{k_0} + w_{k_1}}
  \sum_{j \in [d]} (\mu^j_{k_0} - \mu^j_{k_1})_+^2 $$
where $(\cdot)_+^2: u\mapsto \left(\max\{0, u\}\right)^2$ stands for the squared
positive part.

Computing the cost of shadowing a point by another and conversely is then done
in $O(d)$ operations, which we need to do for each pair of points in the Pareto
set leading us to a computation cost of $O(p^2d)$ for removing a point. We
highlight the difference between $K$ and $p$ as for large values of $K$, $p$
might be significantly lower than $K$.
\end{proof}

We present in Appendix~\ref{app:speedup} a speed-up for the dimension two from
time $O(p^2)$ to time $O(p)$, and we discuss why this speed-up is not possible
in higher dimensions $d>2$.

\subsection{Adding a point to the Pareto set}
The cost of adding a point to the Pareto set doesn't have a closed expression as
was the case for removing a point. It is a more tedious procedure, which shows
in the final algorithmic complexity.

\begin{lemma}[Cost of adding a point to the Pareto set]\label{lem:adding}
  Consider the minimum transportation cost to add any point to the Pareto set
  \begin{equation}
    \begin{array}{cc}
      \min & D_w(\mu, \lambda) \;.\\
      \text{w.r.t.} & \lambda \in \bigcup_{k_0\notin p(\mu)} \Alt^{\add}(k_0)
    \end{array}
  \end{equation}
  The value and minimizer can be computed in time
  $$ O\left(\big(K(p+d) + d^3p\big) \binom{p + d - 1}{d - 1}\right)\,. $$
\end{lemma}
To prove this lemma, we start by looking at how, given a target location $\lambda_0$
for a point $k_0 \notin p(\mu)$, points from the Pareto set should move to ensure
that they are no longer dominating $\lambda_0$. We then provide an algorithm to range
over all the different possible ways points from the Pareto set can move. Then,
given a way the points from the Pareto set move, we compute the minimal
transportation cost and the associated minimizer coherent with this way of
moving points from the Pareto set.

Let $k_0 \notin p(\mu)$. We want to find a $\lambda$ such that in $\lambda$, all points from
$p(\mu)$ are no longer dominating $k_0$ as outlined before. For ease of notation,
we use $0$ as a stand-in for $k_0$ in our subscripts.

We study for now a weaker version of our problem, see Figure~\ref{fig:skthree}.
Given a point $\lambda_0$ to which we will transport $\mu_0$, what is the minimal cost
to move the points from $p(\mu)$ to break their domination of $k_0$. This question
is independent for all points in the Pareto set as they are not interacting with
each others costs, so we can treat them one by one. For a given point $k$ of the
Pareto set, we are only interested in making one of its coordinates below the
corresponding entry of $\lambda_0$. Now, either moving $\mu_0$ to $\lambda_0$ already pushed
$\mu_k$ outside of the upper orthant of $k_0$ or we need to get it outside of the
way. In the first case, the transportation cost is zero, in the other one, as
our transportation cost is separable alongside every dimension, we only need to
find the dimension alongside which the transportation cost is the smallest i.e.
$\frac{w_k}{2} \min_j (\mu_k^j - \lambda_0^j)^2$. We can put the two expressions
together as $\frac{w_k}{2} \min_{j \in [d]} (\mu_k^j - \lambda_0^j)_+^2$ where $(\cdot)_+^2$
is the square of the positive part which is a convex, non-decreasing and
differentiable function. Putting all of the points back together, we get that
given $\lambda_0$, the minimal transportation cost to move everyone from $p(\mu)$
outside of its top right orthant is
$$ g_{k_0}(\lambda_0) \triangleq \frac{w_0}{2} ||\mu_0 - \lambda_0||^2
  + \sum_{k \in p(\mu)} \frac{w_k}{2} \min_{j \in [d]} (\mu_k^j - \lambda_0^j)_+^2. $$

\begin{figure}[ht]
  \vspace{.3in}
  \centerline{\includegraphics[width=.7\linewidth]{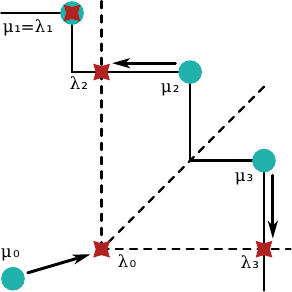}}
  \vspace{.3in}
  \caption{
    The original model is drawn in turquoise (circle). We start by moving the
    point $0$ to a new location. Then we move the points that are still in its
    all-positive orthant outside of it with respect to the dimension where the
    move is the shortest (brown stars).
  }
  \label{fig:skthree}
\end{figure}

Finding the smallest transportation cost is then equivalent to finding the
minimal value of $g$ which we now set out to do. First, it is possible to move
all of the $\min_j$ out to the start of the function giving us $g_{k_0}(\lambda_0) =
\displaystyle{\min_{\varphi: p(\mu) \to [d]} g_{k_0,\varphi}(\lambda_0)}$ where $g_{k_0,\varphi}$ has the
same expression as $g_{k_0}$ but with $\min_j \mu_k^j - \lambda_0^j$ replaced by
$\mu_k^{\varphi(k)} - \lambda_0^{\varphi(k)}$. Hence $\inf_{\lambda_0\in\R^d} g_{k_0}(\lambda_0) = \min_\varphi
\inf_{\lambda\in\R^d} g_{k_0,\varphi}(\lambda_0)$. The $g_{k_0,\varphi}$ are differentiable strictly
convex functions, which makes them quite easy to minimize.

The idea is now to range over all different $\varphi : p(\mu) \to [d]$, to compute the
minimizer of $g_{k_0,\varphi}$ and to take the minimum over all of those. However
following this procedure would lead to a computation cost in $\Omega(d^p)$ which is
exponential in the number of arms in the Pareto set and thus would for most use
case represent to high of a computation cost to reasonably use the algorithm.
Our main insight and the key to tractability is that it is not necessary to
consider all elements in $[d]^{p(\mu)}$. Instead, it turns out that we need only
look at a subset (which depends on the bandit $\mu$) of size $O(p^{d-1})$
polynomial in the number of arms. To understand why, let $\lambda_0$ and let $\varphi: p(\mu)
\to [d]$, the mapping from $k \in p(\mu)$ to $\arg \min_j \mu_k^j - \lambda_0^j$. The set
$S(\varphi)$ of $\lambda_0$ that yield the same $\varphi$ map is given by the following linear
system
$$ \forall k \in p(\mu), \forall j \in [d],
  \quad \mu_k^{\varphi(k)} - \lambda_0^{\varphi(k)} \leq  \mu_k^j - \lambda_0^j\;. $$
For all $k_0 \notin p(\mu), \lambda_0 \in S(\varphi)$, we know that $\forall \varphi':[p]\to[d],\,
g_{k_0,\varphi}(\lambda_0) \leq  g_{k_0,\varphi'}(\lambda_0)$. Given a $\varphi$, $S(\varphi)$ is called the cell
associated with $\varphi$. However, while for any point $\lambda_0$ there is a cell $S(\varphi)$
that contains it, the converse is not the case and a cell associated with a $\varphi$
might well be empty. A cell is empty if
$$  \forall \lambda_0,\, \exists \varphi':\, g_{k_0,\varphi'}(\lambda_0) < g_{k_0,\varphi}(\lambda_0) \;. $$
A $\varphi$ map associated with a non-empty cell is called valid. Ranging over the $\varphi$
with an empty cell is useless, thus, we can restrict our $\min_\varphi$ to valid $\varphi$.

We highlight the fact that while $S(\varphi)$ is the subset of $\R^d$ is the set of
points where $g_{k_0,\varphi}$ is lower than all other $g_{k_0,\varphi'}$, the minimizer in
$\lambda_0$ of $g_{k_0,\varphi}$ might not leave with $S(\varphi)$. However, this is not a cue to
consider the constrained problem where $\lambda_0$ is restricted to live in $S(\varphi)$ as
studying the unconstrained problem would still yield the correct overall
minimizer.

Also, note that $S(\varphi)$ doesn't depend on $k_0$ neither through $\mu_{k_0}$ or
$w_{k_0}$ but only on $\lambda_0$ as such a cell being empty or not doesn't depend on
the point that we might be currently considering. Hence, we can start by
enumerating all non-empty cells and then for each of them we compute the
minimizer of $g_{k_0, \varphi}$ for every different $k_0 \in p(\mu)$, which avoids us
enumerating $K$ different times the non-empty cells. We thus introduce
$$ g_\varphi = \min_{k_0 \notin p(\mu)} g_{k_0, \varphi} ~\text{and}~
  g = \min_{\varphi: S(\varphi) \neq \O} g_\varphi $$
and we get that
$$ \inf_{\lambda_0 \in \R^d} g(\lambda_0) = \min_{\varphi: S(\varphi) \neq \O} \min_{k_0 \notin p(\mu)}
  \inf_{\lambda_0\in\R^d} g_{k_0, \varphi} (\lambda_0) \;. $$

In the next section, we give an algorithm to find non-empty cells and we provide
an analysis on the number of them and the algorithmic complexity of our
algorithm to range over non-empty cells.

\subsubsection{Constructing cells}
Using the observations from the previous section, to know if a cell is empty or
not, we could just range over all possible $\varphi: p(\mu) \to d$ and when one of them
has a non-empty $S(\varphi)$ we compute the minimizer of $g_\varphi$. Let $U\subseteq
V\subseteq p(\mu)$ and $\varphi^r:U \to [d]$ and $\varphi:V\to [d]$. First, we provide a new
altered definition for $S(\varphi)$ which is still compatible with the first one, but
which now works with maps with a restricted domain:
\begin{align*}
  S(\varphi) = \left\{ \lambda_0 \in \R^d \middle|
  \begin{aligned}\forall k \in \dom&(\varphi), \forall j\in [d], \\
    &\mu_k^{\varphi(k)} - \lambda_0^{\varphi(k)} \leq  \mu_k^j - \lambda_0^j
  \end{aligned}\right\}.
\end{align*}
We know assume that $\left.\varphi\right|_U = \varphi^r$, we have that $S(\varphi) \subseteq
S(\varphi^r)$ as we only further constrain the set of equations that defines $S(\varphi^r)$
to construct $S(\varphi)$. This leads to some important results for us as if $\varphi^r$ is
not valid then so is $\varphi$. This prompts us to think of the different $\varphi$ as
leaves of tree for which internal nodes are restricted $\varphi$ maps.

More formally, given an order on the points from the Pareto set $\{k_i | i \in
[p]\} = p(\mu)$, the root of our tree is the empty map, $\varphi: \O \to [d]$. It has
$d$ possible children $\varphi: k_1 \in \{k_1\} \mapsto j$ for all $j\in[d]$. For any
of these children $\varphi_1 \in [d]^{\{k_1\}}$, they themselves have $d$ possible
children given by $\left.\varphi_2\right|_{\{k_1\}} = \varphi_1$ and the different possible
value for $\varphi_2(k_2)$. We continue this process until all the points are
exhausted and the leafs of this tree are the element of $[d]^{[p]}$.

Given this construction and the observation made previously that $S$ is
decreasing along the branches of this tree, we propose a recursive backtracking
algorithm to enumerate non-empty cells. We start from the root and operate a
depth-first search algorithm. When visiting any internal node, we check whether
$S(\varphi)$ is empty or not. If it is not empty, we go on with our search until we
hit a leaf yielding a non-empty $\varphi$ map. Otherwise, we know that any leaves
below it will be associated with an empty cell, as such there is no use in
visiting this subtree at all and we can backtrack one level up and continue our
search in a different subtree.

\begin{figure}[ht]
  \vspace{.3in}
  \begin{center}
    \begin{subfigure}[t]{0.4\linewidth}
      \centerline{\includegraphics[width=\linewidth]{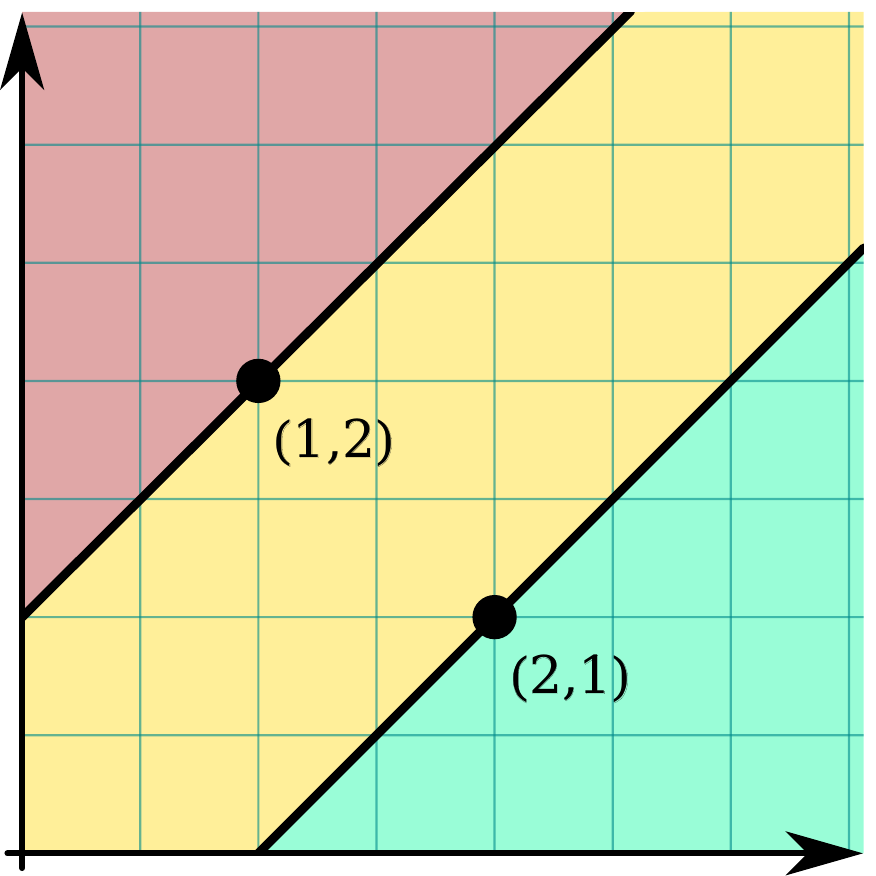}}
    \end{subfigure}~
    \begin{subfigure}[t]{0.528\linewidth}
      \centerline{\includegraphics[width=\linewidth]{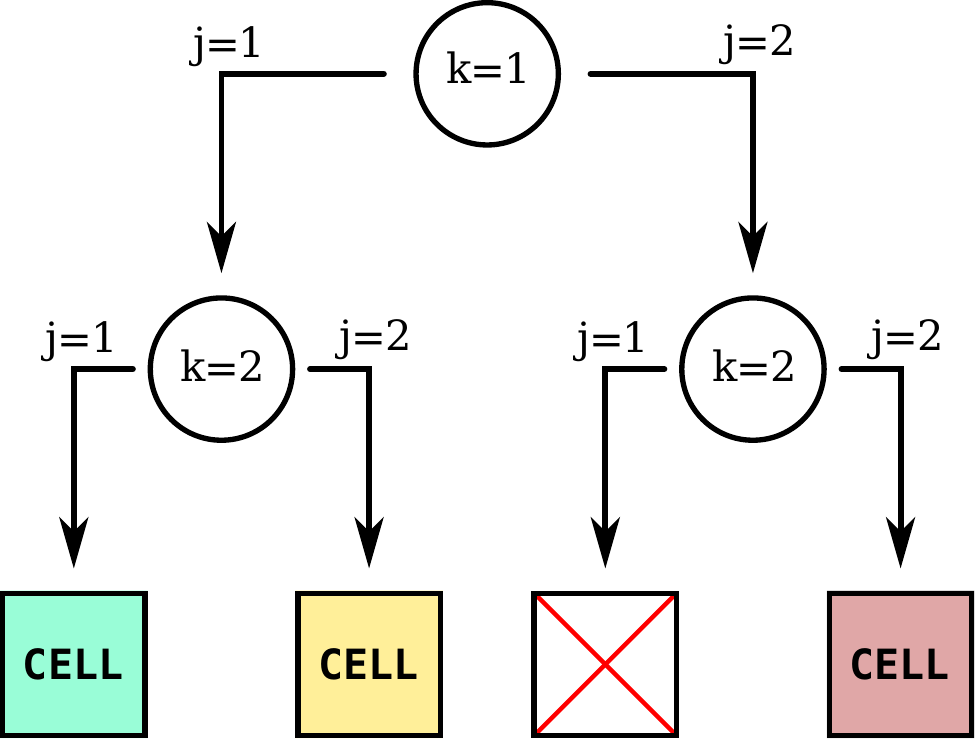}}
    \end{subfigure}
  \end{center}
  \vspace{.3in}
  \caption{
    An example of cell construction in 2d with three valid cells and an empty
    one
  }
  \label{fig:2dexample}
\end{figure}

\begin{example}
Let $\mu_1 = (1,2)$ and $\mu_2 = (2,1)$. The root of our tree is the point $\mu_1$ and
we try to move it alongside dimension $1$ by walking the leftmost edge from the
root. There, we check that our linear system still admits solutions
$$ \mu_1^1 - \lambda_0^1 \leq  \mu_1^2 - \lambda_0^2 ~\Rightarrow~
  \lambda_0^2 - \lambda_0^1 \leq  \mu_1^2 - \mu_1^1 = 1 $$
  which it does so we can go on. Now, we add the second point alongside the
first coordinate which adds $\lambda_0^2 - \lambda_0^1 \leq  \mu_2^2 - \mu_2^1 = -1$ to our set of
equations, which still admits a solution. Since this was the last point, this
means that $\varphi: 1;2 \mapsto 1$ was a non-empty cell and we can go on. We now
backtrack to try and add the second point alongside the second coordinate. This
adds $\lambda_0^1 - \lambda_0^2 \leq  \mu_2^1 - \mu_2^2 = 1$ to our set of linear equations. This
gives us $-1 \leq  \lambda_0^2 - \lambda_0^1 \leq  1$ which admits a solution and thus we found a
new non-empty cell. We now backtrack all the way to our root resetting our list
of inequalities and we add the first point alongside the second coordinate
($\lambda_0^1 - \lambda_0^2 \leq  \mu_1^1 - \mu_1^2 = -1$). But then when we want to add the second
point alongside the first coordinate, we end up with the infeasible system
$\lambda_0^1 - \lambda_0^2 \leq  -1 \wedge \lambda_0^1 - \lambda_0^2 \geq  1$. We can thus discard that sub-tree
and backtrack one step. Here the tree has a small depth meaning we are not
discarding much, but for a bigger tree it could lead to removing a lot of
possible empty cells from our exploration. After that, we continue try to add
the second point alongside its second coordinate which leads to a new cell. So
in this example, there are four possible $\varphi$ of which three are a non-empty
cell.
\end{example}

Moreover, if we consider the $d$ possible children of a node $\varphi^r$ (we assume
that the mapping of $k \in p(\mu)$ is decided at this point in the tree, and we
label $\varphi_j$ the children), the cell of each $\varphi_j$ can be obtained by adding
constraints to $\varphi^r$. If we denote by
$$ C_j \triangleq \left\{\lambda \in \R^d\,\middle|\,\forall j' \in [d],\,
  \mu_k^j - \lambda^j \leq  \mu_k^{j'} - \lambda^{j'}\right\}\;, $$
then $S(\varphi_j) = S(\varphi^r) \cap C_j$. But, since the $C_j$ provide a tesselation of
$\R^d$ (a tesselation being a set of closed sets whose union is $\R^d$ and whose
interiors are disjoints; it is a partition of the space up to the boundaries of
the parts), the $(S(\varphi_j))_{j \in [d]}$ are themselves a tesselation of $S(\varphi^r)$.
When iterating this result, we get that, first a valid internal node will have
a valid internal child and thus a valid internal leaf within its subtree, and
that all (valid) leafs provide a tesselation of $\R^d$ (this is true for any
maximal anti-chain of nodes within the tree) since the root has a cell spanning
over all of $\R^d$. This will be useful later for bounding the number of valid
nodes within the intersection.

The following lemma is proved in Appendix~\ref{app:fast.tree.enums}.
\begin{lemma}\label{lem:tree.enums}
  Checking whether any node $\varphi^r \in [d]^r$ is valid, i.e.\ whether its cell
  $S(\varphi^r)$ is non-empty, can be done in time $O(r + d^3)$. Moreover, by sharing
  computation we can check each of the $d$ extensions to $\varphi^{r+1}$ in time
  $O(d^2)$ each.
\end{lemma}

In Appendix~\ref{app:counting} we prove the following upper bound for the number
of non-empty cells:
\begin{lemma}\label{lem:counting}
  The number of cells is bounded by $\binom{p+d-1}{d-1}$ where $p$ is the number
  of points from the Pareto set.
\end{lemma}

\subsubsection{Finding the optimum within a cell}
We now assume that we reached a leaf of our tree and thus found a valid $\varphi$ map
and we set out to minimize $g_{k_0,\varphi}$ for each different $k_0$. This means that
we fixed the direction in which each point from the Pareto set will move and
given this, we want to find the smallest transportation cost to add the point
$k_0$ to the Pareto set. We recall that $g_{k_0,\varphi}$ has the following
expression:
$$ g_{k_0, \varphi}(\lambda_0) \triangleq
  \frac{w_0}{2} ||\mu_0 - \lambda_0||^2
  + \sum_{k \in p(\mu)} \frac{w_k}{2} \left(\mu_k^{\varphi(k)} - \lambda_0^{\varphi(k)}\right)^2_+. $$

Since the $(\varphi^{-1}(j))_{j\in[d]}$ partitions $[K]$, this function can be
rewritten as a sum of $d$ different function $(h_j)_{j\in[d]}$ such that $h_j$
only depends on $\lambda_0^j$.
$$ h_j : \lambda_0^j \in \R \mapsto \frac{w_0}{2} (\mu_0^j - \lambda_0^j)^2
  + \sum_{k \in \cal \varphi^{-1}(j)} \frac{w_k}{2} (\mu_k^j - \lambda_0^j)^2_+. $$

Thus minimizing each $h_j$ independently is equivalent to minimizing $g_{k_0,
\varphi}$. Moreover, each $h_j$ is a strongly differentiable function thus it is
minimized at $\lambda_0^{*j}$ which is such that $h_j'(\lambda_0^{*j}) = 0$. For the rest of
this section, we will assume that $\varphi^{-1}(j) = [p_j]$ and that $\mu_1^j \leq  \cdots \leq 
\mu_{p_j}^j$. We label for $k \in [p_j],\,x_k = \mu_k^j$ and $x_0 = -\infty,
x_{p_j+1}=+\infty$. For $k\in \{0,\cdots, p_j+1\}$, for $x \in [x_k, x_{k+1}]$,
$$ h'_j (x) = w_0 (x - \mu_0^j) + \sum_{i = k+1}^{p_j} w_i (x - \mu_i^j) \;.$$
As the function admits a unique minimizer, there is only one $k \in \{0,\cdots,
p_j\}$ such that
$$ x_k
  \leq  \frac{w_0 \mu_0^j + \sum_{i = k+1}^{p_j} w_i \mu_i^j}
    {w_0 + \sum_{i=k+1}^{p_j} w_i}
  \leq  x_{k+1} \;.$$
The function $h_j$ is minimized at this point and finding this $k$ is done by
dynamic programming with $O(p_j)$ operations. This is faster than trying a
binary search approach as this would require $O(p_j \log p_j)$ operations. As we
need to compute the minimizer of $h_j$ for each $j$, the computational
complexity to add a point to the Pareto set is of order $O(p+d)$. This needs to
be done for all non-Pareto optimal points, for a total of $O(K(p+d))$
operations.

We now come back to the assumption we made that the $\mu_k^j$ are sorted and
filtered with respect to $\varphi^{-1}(j)$. Doing this for each $k_0$ within each cell
would incur a multiplicative cost of $O(p+p_j\log p_j )$. However sorting can be
performed prior to building our tree by sorting $(\mu_k^j)_{k\in p(\mu)}$ for each
$j$. This has a cost of $O(d p \log p)$ which is negligible when compared to the
rest of our algorithm. Within a cell we can then filter the sorted array to only
get $(\mu_k^j)_{k\in \varphi^{-1}(j)}$ sorted in $O(p)$. As this sequence is common for
all points $k_0$ that we might want to add in a cell this doesn't incur a cost
every time we would like to add a point to the Pareto front, but just once per
cell. This means that for one cell, the complexity to compute the cost to add
each point to the Pareto set is $O(pd + K(p+d))$ which is just $O(K (p+d))$.
Pulling from Lemma~\ref{lem:counting} our upper bound on the number of cells, we
get that this operation is done at most $\binom{p+d-1}{d-1}$ times, and building
the tree yields a final cost of
$$ O\left(\left(K(p+d) + pd^3\right) \binom{p+d-1}{d-1} \right) $$
as we had set out to prove in Theorem~\ref{thm:cplx}.

As most of the settings we aim to tackle can have a large number of arms but
only a few dimensions, this algorithmic complexity boils down to $O(K d^3 p^d)$.

We present in Appendix~\ref{app:speedup} a speed-up for the dimension two from
time $O(Kp^2)$ to time $O(Kp+p\log p)$.

\section{EXPERIMENTS}
We check the performance of our algorithm against the real-world scenario
proposed by~\cite{kone2023adaptive}. We revisit one of their experiment which is
based on the study by~\cite{Munro2021} about immunogenicity of a Covid vaccine
third dose (see the reference for details on the dataset). The setting is a
bandit model $\nu$ of $K=20$ Gaussian arms in dimension $d=3$ representing three
different immunogenicity responses. There are $p(\nu) = 2$ Pareto optimal arms
that we need to identify. The means of the arms and the variance of each
immunogenicity trait can be found in Appendix~\ref{app:experiment}.

For this instance, the instance-dependent factor in the sample complexity lower
bound is $T^*(\nu) = 2103.78$. The associated optimal weights are included in
Appendix~\ref{app:experiment}.

We use a risk of $\delta = 0.1$ and tested the sample complexity average over 2000
runs of the algorithm. Our average sample complexity is 17909. This is
significantly lower than the 39000 sample complexity from 0-APE-$20$ which is
the one corresponding to our setting. We use the stylized exploration rate
$\beta(t,\delta) = \ln\left(\frac{\ln(1+t)}{\delta}\right)$ as our threshold for the stopping
statistic. This is standard in experiments, and though less rigorous than the
choice made in \cite{kone2023adaptive}, it is still overly prudent: though the
risk parameter $\delta$ is set to $0.1$, the real risk is much smaller as we
obtained no identification mistake over the 2000 runs. This conservative
behaviour of pure exploration algorithms has been frequently reported. To speed
up the running time of our algorithm, we throttle the number of minimum
transportation cost computations. While we perform a gradient ascent step for
every sample, we only update the gradient every ten samples and we check the
stopping statistic every 25 samples. The experiment took 5 hours to run on 16
(of 20) cores of a dual Intel(R) Xeon(R) CPU E5-2630 v4 machine.

In Figure~\ref{fig:graph}, we show the empirical distribution of the stopping
time of our algorithm. We highlight that there is a significative difference
between $\E_\theta(\tau_\theta)$ and $\log\left(\frac{1}{\delta}\right)T^*(\theta)$, whereas we could
have expected them to be closer to each other. Several factors contribute to
this gap. We chose a too large value of $\delta$ to exhibit the expected asymptotic
behaviour. The analysis shows the presence of a second-order term that is not
neglictible on such experiments. It is all the more significant that we chose
not to update the gradient at every round but only every few rounds.

\begin{figure}[ht]
  \vspace{.3in}
  \centerline{\includegraphics[width=1.15\linewidth]{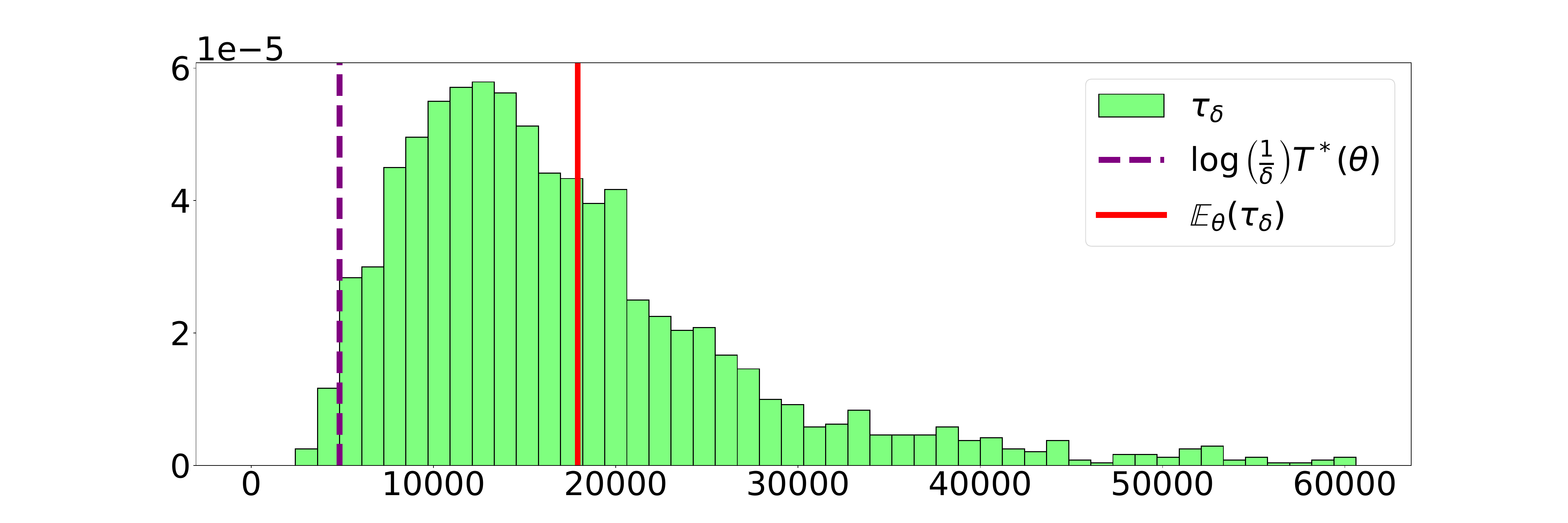}}
  \caption{
    Empirical distribution of the number of samples used to identify the two
    Pareto optimal points
  }
  \vspace{.3in}
  \label{fig:graph}
\end{figure}

This experiment is representative of a real edge of Track-and-Stop in terms of
sample complexity. This advantage can be theoretically understood by comparing
the complexity bounds when $\delta$ goes to $0$. For example, in the scenario
when one arm dominates a large number of other identical arms, Track-and-Stop
can be proved to be more data efficient by a factor almost~2.

We also ran experiments on random instances to evaluate the computation time of
the minimization solver and its dependency on $d$ and $p$. The random instances
on which we ran our algorithm consisted of $p$ points sampled from the
all-positive quadrant of a $d$ dimension sphere (as such they are all easily
part of the Pareto set) and an additional point at $0$. We only ran the
minimzation solver of our algorithm on these points and estimated the time it
took for each pair of $(p,d)$ on 100 samples. The result of this experiment is
represented in Figure~\ref{fig:complexityexp}. In the Figure, we can see that
given a fixed $d$ there is a linear dependency between the log of the time taken
and the log of the number of point on the Pareto set. The slope of each line is
proportional to $d$. This matches the theorical result obtained in
Theorem~\ref{thm:cplx} as for a given $d$ and $K=p+1$ the complexity of our
solver should be $O(p^{d+1})$.

\begin{figure}[ht]
  \vspace{.3in}
  \centerline{\includegraphics[width=\linewidth]{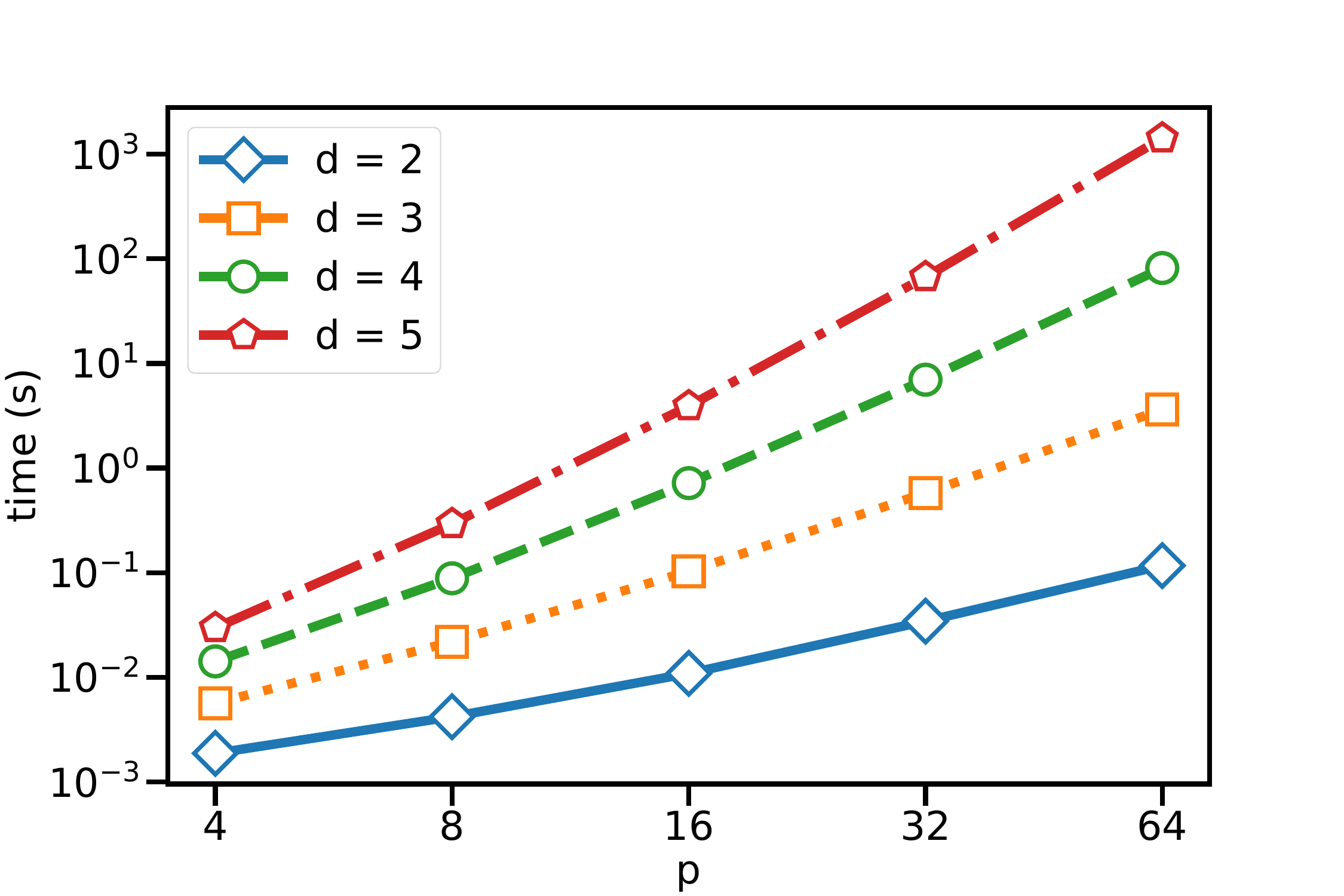}}
  \vspace{.3in}
  \caption{
    Time to solve the minimization problem on a random point cloud with $p$
    Pareto points in dim. $d$
  }
  \label{fig:complexityexp}
\end{figure}

\section{CONCLUSION}
We tackled the problem of Pareto front identification in a Gaussian multi-armed
bandit. To this end, we studied efficient implementation of the core oracle
required by the Track-and-Stop framework, namely the gradient of the
information-theoretic lower bound. To solve the associated non-convex
optimization problem, we split the domain in convex parts, discussed enumerating
the parts and solving the convex problem on the parts in closed form.

For future work we are interested in relaxing the assumptions. In particular, we
aim to study the problem under dependent coordinates, with Gaussians of unknown
variances, in other exponential families, in non-parametric classes, and in the
approximate $\epsilon > 0$ case. It would be interesting, challenging and
rewarding to pin down the computational complexity of the transportation problem
\eqref{eq:arginf}, already in the spherical Gaussian case. Can one find and
exploit additional structure in the problem to solve it in time at most a fixed
and dimension independent degree polynomial in the number of arms $K$? Or can
one prove a lower bound matching Theorem~\ref{thm:cplx}?

\bibliography{biblio.bib}

\clearpage
\section*{Checklist}
\begin{enumerate}
  \item For all models and algorithms presented, check if you include:
  \begin{enumerate}
    \item A clear description of the mathematical setting, assumptions,
      algorithm, and/or model. Yes
    \item An analysis of the properties and complexity (time, space, sample
      size) of any algorithm. Yes
    \item (Optional) Anonymized source code, with specification of all
      dependencies, including external libraries. Yes
   \end{enumerate}

  \item For any theoretical claim, check if you include:
  \begin{enumerate}
    \item Statements of the full set of assumptions of all theoretical results.
      Yes
    \item Complete proofs of all theoretical results. Yes
    \item Clear explanations of any assumptions. Yes
  \end{enumerate}

  \item For all figures and tables that present empirical results, check if you
    include:
  \begin{enumerate}
    \item The code, data, and instructions needed to reproduce the main
      experimental results (either in the supplemental material or as a URL).
      Yes
    \item All the training details (e.g., data splits, hyperparameters, how
      they were chosen). Not Applicable
    \item A clear definition of the specific measure or statistics and error
      bars (e.g., with respect to the random seed after running experiments
      multiple times). Yes
    \item A description of the computing infrastructure used. (e.g., type of
      GPUs, internal cluster, or cloud provider). Yes
  \end{enumerate}

  \item If you are using existing assets (e.g., code, data, models) or
    curating/releasing new assets, check if you include:
  \begin{enumerate}
    \item Citations of the creator If your work uses existing assets. Not
      Applicable
    \item The license information of the assets, if applicable. Not Applicable
    \item New assets either in the supplemental material or as a URL, if
      applicable. Not Applicable
    \item Information about consent from data providers/curators. Not Applicable
    \item Discussion of sensible content if applicable, e.g., personally
      identifiable information or offensive content. Not Applicable
  \end{enumerate}

  \item If you used crowdsourcing or conducted research with human subjects,
    check if you include:
  \begin{enumerate}
    \item The full text of instructions given to participants and screenshots.
       Not Applicable
    \item Descriptions of potential participant risks, with links to
      Institutional Review Board (IRB) approvals if applicable. Not Applicable
    \item The estimated hourly wage paid to participants and the total amount
      spent on participant compensation. Not Applicable
  \end{enumerate}
 \end{enumerate}

\appendix
\newpage
\onecolumn 

\section{Proof of Lemma~\ref{lem:splitting}: Splitting the domain}
\label{app:pf.l1}
\begin{proof}[Proof of Lemma~\ref{lem:splitting}]
$\implies$) We assume $p(\mu) \neq p(\lambda)$. Then either $p(\mu) \backslash p(\lambda)$ or
$p(\lambda) \backslash p(\mu)$ is not empty. We assume the first for now. And we let
$k_0 \in p(\mu) \backslash p(\lambda)$. Since $k_0 \notin p(\lambda)$, then its set of
dominators in $\lambda$ is not empty. If any of its dominators in $\lambda$ belongs to
$p(\mu)$, then we have found $k_0,k_1 \in p(\mu)$ such that $k_0 \preceq_\lambda k_1$.
Otherwise let $k_0'$ from $k_0$'s set of dominators in $\lambda$ since it is not
empty. We know that $k_0'$ does not belong to $p(\mu)$ and since all of its
dominators in $\lambda$ is included in $k_0$'s one (by transitivity), then none of its
dominators belongs to $p(\mu)$. Thus, we found $k_0' \notin p(\mu)$ such that it is
not dominated in $\lambda$ by any point from $p(\mu)$.

We now assume that $p(\lambda) \backslash p(\mu)$ is not empty. Then there exists $k_0
\in p(\lambda)$ such that $k_0 \notin p(\mu)$. Let $k_0$ be such an index and since it
belongs in $p(\lambda)$ no one dominates it and in particular points from $p(\mu)$.

$\Longleftarrow)$ We first assume that there exists $k_0, k_1 \in p(\mu)$ such
that $k_0 \preceq_\lambda k_1$. Let $k_0, k_1$ such indices. Since $k_0 \preceq_\lambda k_1$
then $k_0 \notin p(\lambda)$. Hence $p(\mu) \neq p(\lambda)$.

We now assume that there exists $k_0 \notin p(\mu)$ such that no points in $p(\mu)$
dominates it in $\lambda$. Let $k_0$ such an index. Either $k_0$ is now on the Pareto
set and we are done or there exists a point $k$ from the Pareto set of $\lambda$ that
dominates it. But then this point is not within $p(\mu)$ (because no points in
from $p(\mu)$ dominates $k_0$ in $\lambda$) and is in $p(\lambda)$. Hence $p(\mu) \neq p(\lambda)$.
\end{proof}

\section{Proof of Lemma~\ref{lem:tree.enums}: An efficient algorithm for our
  linear system of equations
}\label{app:fast.tree.enums}
To check whether our system of equations admits a solution, we could invoke a
linear programming solver to get an answer. Here we leverage the particular
structure of our set of inequalities to decide feasibility more efficiently.
First note that any of the inequalities that we might add to our system is of
the form $\lambda_0^{j_2} - \lambda_0^{j_1} \leq  \mu_k^{j_2} - \mu_k^{j_1}$. Thus our system of
inequalities might be viewed as a sequence of upper bounds for differences of
coordinates of $\lambda_0$. To check whether this system admits a solution we
introduce a directed multi-graph $G$ with nodes labeled by the $j \in [d]$. For
each constraint of the form $\lambda_0^{j_2} - \lambda_0^{j_1} \leq  \mu_k^{j_2} - \mu_k^{j_1}$ we
add an edge going from $j_1$ to $j_2$ which has value $\mu_k^{j_2} - \mu_k^{j_1}$.
We use a multi-graph representation because it has a one to one mapping with our
system of equation.

We recall a well-known \cite{erickson2019algorithms} equivalence between the
existence of a solution for a set of equations that can be encoded thusly.
\begin{lemma}
  Let $G = (V, E)$ be a finite directed graph where $E \subseteq V×V×\R$. The
  system
  \[
    \lambda \in \R^d
    \quad \text{s.t.}\quad
    \forall (i,j,v) \in E : \lambda_i - \lambda_j \le v
  \]
  has a solution iff $G$ has no negative cycle.
\end{lemma}

\begin{proof}
  We assume that $G$ contains a negative cycle $j_0 \to_{v_1} \cdots \to_{v_n}
  j_n = j_0$. Then, we know that for all $i \in [n]$ the equation $\lambda_{j_i} -
  \lambda_{j_{i-1}} \leq  v_i$ is present in our linear system. Summing these inequalities
  gives us
  $$ 0 = \lambda_{j_n} - \lambda_{j_0}
    = \sum_{i \in [d]} \lambda_{j_i} - \lambda_{j_{i-1}}
    \leq  \sum_{i \in [d]} v_i < 0 \;.$$
  Hence our system does not admit solutions.

  Now we assume that the graph doesn't contains a negative cycle and we
  introduce a source point $s$  which we connect to every vertex of $G$ with an
  edge of length $0$. This updated graph still doesn't contain negative cycles
  because no new cycles have been created. Thus we can define $\delta: j \in [d] \to
  \R$ the function that gives the length of the shortest path from $s$ to $j$.
  We claim that setting $\lambda_0^j$ to $\delta(j)$ respects every inequality and thus
  is a solution to our linear system of equations. This is due to the fact that
  for every edge $j_1 \to_v j_2$ the triangular inequality gives us that
  $\delta(j_2) \leq  \delta(j_1) + v$.
\end{proof}

As such, we could, at each internal node of our tree, create the graph
associated with the set of equations encountered so far and check the presence
of negative cycles. If there are any we backtrack, otherwise we go on. If we do
that, to check for presence of negative cycle we don't have a much better choice
than the Bellman-Ford algorithm \cite{schrijver2003combinatorial} because of the
presence of negative edges. This would yield a algorithmic complexity of $O(V
E)$ which in our case more often than not would look like $d^3$. We can reduce
that to a $d^2$ by noting that when we make our choice of $(k,j)$ when going
down the tree, the graph at the parent and the child node only differ by a few
edges, the one stemming from node $j$. Thus if we already know the shortest path
in the original graph (which is assumed to not contain negative cycles), we can
leverage that knowledge to speed up our negative cycle detection at the child
node, which we state in the next lemma.

\begin{lemma}
  Let $G = (V,E)$ a directed weighted graph with no negative cycles and such
  that $v: V × V \to\R$ is the length of the shortest path between every pair
  of points. Let $\{u_x\}_{x\in V}$ some new edges between $0$ and $x$ such that
  $G' = (V, E\cap \{u_x\}_{x\in V})$. Negative cycle detection and updated
  shortest paths can be performed in time $O(d^2)$.
\end{lemma}

\begin{proof}
  If the graph $G'$ now contains a negative cycle, it is going through an
  updated edge as $G$ doesn't contain a negative cycle. By induction the
  minimal value of a cycle going from and back to $0$ without visiting any other
  nodes more than once is $\min_{x\neq 0} u_x + s(x, 0)$ where $s(x,y)$ is
  the length of the shortest simple path (a path that never goes twice through
  the same vertex) from $x$ to $y$ in $G'$. But that simple path from $x$ to
  $0$ don't go through any updated edge, hence $s(x,0) = v(x,0)$. Thus if all
  $u_x + v(x,0) \ge 0$ there is no negative cycle in our updated graph and if
  any $u(x) + v(x,0)$ is negative, we found a negative cycle. This can be
  checked in $O(d)$.

  We now assume that there is no negative cycle in $G'$ and set out to compute
  $s(x,y)$ for all $x,y \in V$. First as stated earlier, paths that go to $0$
  haven't changed cost as they cannot go through any of the updated edges. We
  now update edges going from $0$ and claim that $s(0,x) = \min B$ where $B =
  \{_y u_y + v(y,x) \,|\, y \in V\}\cup \{v(0,x)\}$.

  First, we show that $s(0,x) \in B$. Let $p = 0 \to_{e_1} y^* \to \cdots \to
  x$, then either $e_1$ is an edge that was present in $G$ and thus $s(0,x) =
  v(0,x)$ or $e_1$ was not and then it has weight $u_{y^*}$. We call $l$ the
  length of the path $x \to \cdots \to y^*$ in $G'$. Since the path $p$ is the
  shortest between $0$ and $x$ and that there is no negative cycle, the part of
  $p$ between $y^*$ and $x$ doesn't go through $0$ and as such any new edges.
  Since it is the shortest path between $y^*$ and $x$ and it stays in common
  part of $G$ and $G'$, we get that $l = v(y,x)$. Hence the length of this path
  is $u_{y^*} + v(y^*,x)$.

  The shortest path between $0$ and $x$ in $G$ is still a path in $G'$ thus
  $s(0,x) \leq  v(0,x)$. Now for all $y\in V$, either the shortest path from $y$ to
  $x$ in $G'$ doesn't goes though $0$ or it goes through it. In the first case,
  the path $0 \to y \to \cdots \to x$ in $G'$ has value $u_y + v(y,x)$ and since
  it is the length of a path from $0$ to $x$, we know that $s(0,x) \leq  u(y) +
  v(y,x)$. In the second case, we know that the shortest path from $y$ to $x$ in
  $G'$ goes exactly once through $0$ (otherwise there would be a negative
  cycle), and using the optimality of subpath, we know that the part of the path
  from $y$ to $x$ after going through $0$ is the shortest path between $0$ and
  $x$, thus the next point visited is $y^*$ and this part of the path has length
  $s(0,x)$, which is such that $s(0,x) \leq  v(0,x)$. Similarly since the part
  between $y$ and $0$ is the shortest path between these points in $G'$ and that
  it doesn't go through an edge stemming from $0$, we know that it has length
  $v(y,0)$ giving us the final length from this path to be $u(y) + v(y,0) +
  s(0,x)$. Since there is no negative cycles we know that $u(y) + v(y,0) \geq  0$.
  Since all path from $y$ to $x$ in $G$ are still path in $G$' we have that
  $v(y,0) + s(0,x) = s(y,x) \leq  v(y,x)$. This gives us that $$ s(0,x) \leq  u(y) +
  v(y,0) + s(0,x) \leq  u(y) + v(y,x) \;.$$ Hence for all $b \in B$, $b\geq  s(0,x)$ and
  $s(0,x) \in B$, which give us that $s(0,x) = \min B$. Hence all shortest path
  from $0$ to other nodes can be updated in $O(d^2)$.

  To conclude, we only need to update all distances going neither from nor to
  $0$. We claim that for $x$ and $y$ which are not $0$, $s(x,y) = \min \{ v(x,y)
  s(x,0) + s(0,y)\}$. First, we highlight that $s(x, 0) + s(0,y)$ is the length
  of the shortest path from $x$ to $y$ going through $0$ in $G'$ and that
  $v(x,y)$ is the length of the shortest path from $x$ to $y$ in $G$ which is
  still present in $G'$. Thus both of them are path length in $G'$ from $x$ to
  $y$ and we just need to show that either of them has an optimal length. The
  shortest path from $x$ to $y$ in $G'$ either goes through an edge stemming
  from $0$ or it doesn't. In the first case, it has value $v(x,y)$ otherwise it
  has value $ s(x,0) + s(0,y)$. Hence $s(x,y) = \min \{v(x,y); s(x,0) +
  s(0,y)\}$. All these distances can be updated in $O(d^2)$ which concludes our
  proof.
\end{proof}

\section{Proof of Lemma~\ref{lem:counting}: Counting cells}\label{app:counting}
As stated earlier, the number of possible $\varphi$ maps is $d^p$ but of these,
quite a lot don't lead to a cell. To better estimate the complexity of our
algorithm, we would like, with $p$ and $d$ known, to give a tighter estimate of
the number of cells. In this section, we start by showing for small dimensions
how to derive explicitly the number of cells. Then we provide an upper bound for
arbitrary dimension.

\textbf{Dimension 1:} Though it might feel a bit dull, we will start by tackling
the case of best arm identification. As there is only one dimension to choose
from, the number of cells is exactly one.

For higher dimensions, we first start by noting that cells are invariant to
translation by the all-ones vector, as adding the same offset to $\lambda^j_0$ in
every dimension will not change the comparison between the $\mu^j_k - \lambda^j_0$ for
different $j$. Thus, even though cells live in $\R^d$, they can be reduced to
the orthogonal space of the all-ones vector which has dimension $d-1$.

In \textbf{dimension 2}, the direction toward which a point $\mu_k$ should go are
given by whether $\lambda_0$ is above or below the diagonal line stemming from $\mu_k$.
In the reduced space this is equivalent to being right or left of the projected
$\mu_k$. As such there is $p + 1$ different cells: one to the right of every
point, one to the left of the rightmost point and one to the right of all the
other points and so on until all points are on the left.

In \textbf{dimension 3}, all points shatter the space (original and reduced) in
three distinct parts: one where the distance to the first, second or third
dimension is the smallest. In the reduced space, they are delimited by three
lines stemming from the projected point and going to infinity. Using this
representation (shown in Figure~\ref{fig:3d_count}, it is possible to use
Euler's formula to count the number of different cells. To do so, we will count
the number of cells iteratively by adding the $\mu_k$ one by one. We start with no
points, which means we have 1 cell, 0 edges and 0 vertices. When we add the
$k+1^\text{th}$ point, for all the previously added points, one of its edge will
cross one of the previously added edges adding 1 point and 2 edges per point.
This makes the total number of edges and vertices added by adding a new point be
$3 + 2k$ for the edges and $1 + k$. Using Euler's formula, we get that the
number of added cells is given by $3 + 2k - (k + 1) = k + 2$. Hence the number
of cells for $k$ points in the Pareto set is
\begin{equation}\label{eq:3d_count}
  1 + \sum_{i = 1}^k (i-1) + 2 = \frac{k(k+1)}{2}\,.
\end{equation}

\begin{figure}[ht]
  \begin{center}
    \begin{subfigure}[t]{0.3\linewidth}
      \centerline{\includegraphics[width=\linewidth]{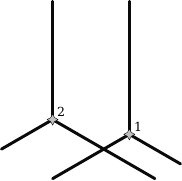}}
      \caption{Initial setting with two points.}
    \end{subfigure}%
    \hspace{0.5cm}
    \begin{subfigure}[t]{0.3\linewidth}
      \centerline{\includegraphics[width=\linewidth]{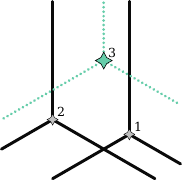}}
      \caption{
        An additional third point is added, adding a vertex and three edges.
      }
    \end{subfigure}%
    \hspace{0.5cm}
    \begin{subfigure}[t]{0.3\linewidth}
      \centerline{\includegraphics[width=\linewidth]{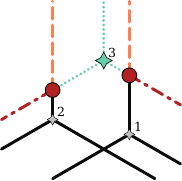}}
      \caption{
        For every previously added point, one of the added edges intersects with
        an old edge adding one vertex and two edges.
      }
    \end{subfigure}
  \end{center}
  \caption{Example of a point being added in dimension 3}
  \label{fig:3d_count}
\end{figure}

\begin{proof}[Proof of Lemma~\ref{lem:counting}]
  To show this lemma, we reuse our graph representation from the previous
  section. For a specific $(j_1,j_2)$, the list of edges going from $j_1$ to
  $j_2$ in our multi graph $G$ is given by $\{ \mu_k^{j_2} - \mu_k^{j_1} \,|\, k \in
  \cal \varphi^{-1}(j) \}$ and for a specific node $k$ such that $j_1 = \varphi(k)$ an edge
  is created between $j_1$ and every other $j_2 \neq j_1 \in [d]$ with value
  $\mu_k^{j_2} - \mu_k^{j_1}$.

  First, let $v(\varphi) = \left(\text{Card}(\varphi^{-1}(j))\right)_{j \in [d]}$. We want
  to show that for a specific value of $v(\varphi)$, there is at most one possible $\varphi$
  which leads to a valid cell. To do that, we show that for $\varphi, \varphi'$ such that
  $v(\varphi) = v(\varphi')$, there exists a particular permutation $\pi$ such that $\varphi' =
  \varphi\circ \pi$. Then we show that if there were no negative cycles in $G$, the
  multi-graph associated with $\varphi$, then the transformation of $G$ to $G'$ after
  having applied the permutation creates a negative cycle. We conclude by saying
  that the cardinality of the image of $v$ is $\binom{p+d-1}{d-1}$.

  Let $\varphi, \varphi'$ such that $v(\varphi) = v(\varphi')$ but $\varphi \neq \varphi'$. For each $j \in [d]$ we
  let $\pi_j: \varphi'^{-1}(j) \to \varphi^{-1}(j)$ be a bijection between its domain and its
  image (there exists one since they have the same cardinal $v(\varphi)_j =
  v(\varphi')_j$). We now define $\pi:
  k \in p(\mu) \to \pi_{\varphi'(k)}(k)$. Let $k_0,k_1$ such that $\pi(k_0) = \pi(k_1)$, since
  the $\varphi^{-1}(j)$ we know that there exists a unique $j$ such that
  $\pi(k_0) \in \varphi^{-1}(j)$, thus $k_0, k_1 \in \varphi'^{-1}(j)$ and $\pi_j(k_0) =
  \pi_j(k_1)$ but since $\pi_j$ is a bijection between $\varphi'^{-1}(j)$ and $\varphi^{-1}(j)$
  we get that $k_0 = k_1$. Hence $\pi$ is injective. And since $\pi$'s domain and
  image has the same cardinal, it is bijective. Now, for all $k \in p(\mu)$ with
  $j = \varphi'(k)$, $\pi(k) \in \varphi^{-1}(j)$ hence $\varphi \circ \pi (k) = j = \varphi'(k)$. Thus, for
  all $\varphi, \varphi'$ such that $v(\varphi) = v(\varphi')$ there exists a permutation $\pi$ such that
  $\varphi' = \varphi \circ \pi$.

  We now decompose $\pi$ in disjoints cycles. If a cycle $c = (k_1, \cdots, k_n)$
  is such that $\varphi(k_1) = \cdots = \varphi(k_n)$ then this cycle can be omitted from
  the permutation and it would yield the same $\varphi'$. If all cycles were to be
  removed this way, we would end up with $\varphi = \varphi' \circ Id = \varphi'$, thus we can
  conclude that at least one cycle as to be such that $\text{Card}\left(\varphi\{k_1,
  \cdots, k_n\}\right) > 1$. We now restrict ourselves to $\pi$ with only such
  cycles as other can be safely deleted without changing the resulting $\varphi'$. We
  now focus on a cycle $c=(k_1, \cdots, k_n)$ such that there exists $x < y \in
  [n]$ such that $\varphi(k_x) = \varphi(k_y)$ and we prove that there exists two cycles
  $c_x, c_y$ by which $c$ can be replaced in $\pi$ which yields the same $\varphi'$.
  Since $\varphi(k_x) = \varphi(k_y)$ then $\varphi\circ\pi = \varphi\circ (k_x k_y) \circ\pi$, and $(k_x
  k_y) \circ \pi$ is permutation that has the same cycle than $\pi$ except for $c$
  which has been split in two disjoints cycles (they are disjoint from each
  other and from cycles in $\pi$ which are not $c$): $c_x = (k_x, \cdots,
  k_{y-1})$ and $c_y = (k_y, \cdots, k_{x-1})$. Thus $\pi$ and $(k_x k_y)\circ \pi$
  yield the same $\varphi'$, and we can restrict ourselves to permutation such that
  all $\varphi(k_i)$ are distinct within each different cycle.

  We let $\pi$ be such a permutation and we introduce $G$, the multi-graph
  associated with $\varphi$ and $G'$ the one for $\varphi' = \varphi \circ \pi$. We now assume that
  $\varphi$ is a valid map and we will restrict the rest of our study on the changes
  operated on $G$ by a specific cycle $c = (k_i)_{i \in [n]}$. Let $i \in[n],
  j_1= \varphi(k_i)$ and $j_2 = \varphi'(k_i) = \varphi(k_{i+1})$. Since $j_1 \neq j_2$, we know
  that there exists an edge in $G$ between $j_1$ and $j_2$ of value $v_i
  \triangleq \mu_{k_i}^{j_2} - \mu_{k_i}^{j_1}$, respectively in $G'$ we know that
  there is an edge between $j_2$ and $j_1$ of value $\mu_{k_i}^{j_1} -
  \mu_{k_i}^{j_2} = -v_i$. As $\varphi$ is assumed to be a valid map, the cycle from $G$
  $$ \varphi(k_1) \to_{v_1} \cdots \to_{v_{n-1}} \varphi(k_n) \to_{v_n} \varphi(k_1) $$
  is not negative. As such, the cycle from $G'$ given by
  $$ \varphi'(k_n) \gets_{-v_1} \varphi'(k_1) \gets_{-v_2} \cdots
    \gets_{-v_n} \varphi'(k_n) $$
  is a negative cycle. Thus there is at most one valid map for every point of
  the image of $v$.

  The image of $v$ is the set of vectors of length $d$ whose entries are natural
  numbers that sum to $p$. It is a well-known result that this set has
  cardinality $\binom{p+d-1}{d-1}$, which concludes the proof.
\end{proof}

We believe that this bound is in fact the exact cell count (as seen for
dimensions up to $3$, cf Eq.\ref{eq:3d_count} and by simulations for $p,d$ up to
$11$) but we settle for this upper bound in the analysis of the complexity. What
this tells us is that the number of leaves within the tree is at most
$\binom{p+d-1}{d-1}$. Using the observation that if an internal node is not
empty there always exists a non-empty leaf below it, we know that the number of
valid internal nodes at a given depth is bounded by $\binom{p+d-1}{d-1}$. Thus,
the number of non-empty internal nodes is bounded by $p\binom{p+d-1}{d-1}$. Each
of these non-empty internal nodes may have at most $d$ children which can be
non-empty or empty internal or non-empty or empty leaves. As we run our cell
elimination procedure for each child of each non-empty internal node, we run our
procedure at most $d p \binom{p+d-1}{d-1}$ times. Since this procedure has
complexity $d^2$ we end up with a complexity of $d^3 p \binom{p + d-1}{d-1}$ to
construct our tree.

\section{Speed-up for Dimension $2$}\label{app:speedup}
In this section we present a speed-up in run time available for dimension $d=2$.
The transportation cost computations both for removing a point from the front
and for adding a point to the front are presented next.

\subsection{Removing a point}
We present an update of Lemma~\ref{lem:removing} for dimension $2$. Here, we can
leverage the geometry of our problem to disregard most of the pairs $(k_0,
k_1)$. Let's assume that we are considering the pair $(k_0, k_2)$ and that there
exists an arm $k_1$ in the Pareto set between $k_0$ and $k_2$ (i.e. $\mu^1_{k_0} \geq 
\mu^1_{k_1} \geq  \mu^1_{k_2}$ and $\mu^2_{k_0} \leq  \mu^2_{k_1} \leq  \mu^2_{k_2}$). Now the inf
reached at $\lambda^1_{k_0} = \lambda^1_{k_2}$ will be either above or below $\mu^1_{k_1}$. If
it is reached below then $\mu^1_{k_1} = \lambda^1_{k_1} \geq  \lambda^1_{k_0}$ which means the
cost of shadowing $k_0$ with $k_2$ is higher than the cost of shadowing $k_0$
with $k_1$ and if it is above then $\mu^1_{k_1} = \lambda^1_{k_1} \leq  \lambda^1_{k_2}$ which
means the cost of shadowing $k_0$ with $k_2$ is higher than the cost of
shadowing $k_1$ with $k_2$. Hence by ordering the Pareto set, we can restrict
ourselves to only look at adjacent points within the Pareto set as these will
yield the smallest transportation cost. This means that the number of pair that
we need to examine is just $O(p)$ giving us the reduced computation cost $O(p)$
for removing a point.

However, this technique doesn't scale well with higher dimensions. Given three
points $(k_0, k_1, k_2)$, a similar result can be obtained if there exists $j_a$
such that $\mu_{k_0}^{j_a} \geq  \mu_{k_2}^{j_a} \geq  \mu_{k_1}^{j_a}$ and for all $j \neq
j_a,\,\mu_{k_0}^j \leq  \mu_{k_2}^j \leq  \mu_{k_1}^j$. In this setting the minimizer $\lambda$ of
the transportation cost to shadow $k_0$ by $k_1$ is such that either
$\lambda_{k_0}^{j_a} \leq  \mu_{k_2}^{j_a}$ and then $\lambda_{k_0} \preceq \mu_{k_2}$ or
$\lambda_{k_1}^{j_a} ( = \lambda_{k_0}^{j_a} ) \geq  \mu_{k_2}^{j_a}$ and then $\mu_{k_2} \preceq
\lambda_{k_1}$. As previously, we found that computing the value for the pair $k_0,
k_1$ was unnecessary because either the pair $(k_0, k_2)$ or $(k_2, k_1)$ would
have yielded a smaller value. But this was done in a pretty constrained way,
removing the constraint that there is at most one direction alongside which
$\mu_{k_0}^j \geq  \mu_{k_1}^j$ won't lead to any results. To see that, let $j_a, j_b$
such that $\forall j \in \{j_a, j_b\},\,\mu_{k_0}^j \geq  \mu_{k_2}^j \geq  \mu_{k_1}^j$, then
we might end up with $\lambda_{k_0}^{j_a} \leq  \mu_{k_2}^{j_a}$ and $\lambda_{k_0}^{j_b} \geq 
\mu_{k_2}^{j_b}$, leading to no claim of the shape $\lambda_{k_0} \preceq \mu_{k_2}$ or
$\mu_{k_2} \preceq \lambda_{k_1}$.

\subsection{Adding a point}
Here we present an update of Lemma~\ref{lem:adding} for dimension $2$. Again,
the geometry of the Pareto front allows us to speed up the computation of the
minimal transportation cost to add a point to the Pareto set.

We recall the definition of the function $g_{k_0}$ which, given a new location
$\lambda_0$ for the point $\mu_{k_0}$ (labeled $\mu_0$ for ease of notation) gives the
smallest transportation cost to add this point to the Pareto set while moving
it to the new location.
$$ g_{k_0}(\lambda_0) = \frac{w_0}{2}||\mu_0 - \lambda_0||^2 + \sum_{k\in p(\mu)}
  \frac{w_k}{2}\min_{j\in[d]}\left(\mu_k^j - \lambda_0^j\right)^2_+. $$

And given a map $\varphi:p(\mu)\to[d]$, we also recall the function
$$ g_{k_0,\varphi}(\lambda_0) = \frac{w_0}{2}||\mu_0 - \lambda_0||^2 + \sum_{k\in p(\mu)}
  \frac{w_k}{2}\left(\mu_k^{\varphi(k)} - \lambda_0^{\varphi(k)}\right)^2_+. $$

The functions $g_{k_0}$ and $g_{k_0,\varphi}$ are equal on a set $S(\varphi)$ which consists
of the solutions of the following linear system:
$$ \lambda_0\in\R^d\text{ st. }
  \forall k\in p(\mu),\,\forall j \in [d],\; \mu_k^{\varphi(k)} - \lambda_0^{\varphi(k)}
    \leq  \mu_k^j - \lambda_0^j $$
and outside of this set, $g_{k_0} \leq  g_{k_0,\varphi}$. We highlight here that for any
$\lambda_0\in S(\varphi)$, the line generated by $\lambda_0 + t\mathds{1}$ is included in $S(\varphi)$
where $\mathds{1}$ consists of the all-one vector. Hence the geometry of cells
can be reduced to the orthogonal space to $\R\mathds{1}$. Thus from now on, we
will decompose $\lambda_0$ in an $s$ part which lives in $\R^{d-1}$ and a $t$ part
which lives in $\R$, such that $\lambda_0 = Ms+t\mathds{1}$, where $M$ is a $d × d-1$
matrix such that its columns are all
orthogonal to each others and to $\mathds{1}$, and of
norm $1$. We apply the analogous decomposition to $\mu_0$ and $(\mu_k)_{k\in p(\mu)}$
leading to $s_0, t_0$ and $(s_k, t_k)_{k\in p(\mu)}$ and we redefine the functions
$g_{k_0}$ and $g_{k_0,\varphi}$ accordingly:
$$ g_{k_0,\varphi}(s,t) = \frac{w_0}{2}\left(||s_0 - s||^2 + d(t_0 - t)^2\right)
  + \sum_{k\in p(\mu)} \frac{w_k}{2}
    \left(\left(M(s_k - s)\right)_{\varphi(k)} + t_k -t\right)^2_+.$$

For a given $s$ and $\varphi$, we are interested in finding the $t$ that minimizes
$g_{k_0,\varphi}(s,t)$. We label $t^*_\varphi(s)$ and $g_{k_0,\varphi}(s)$ (resp. $t^*(s)$ and
$g_{k_0}(s)$ for the minimizer and the minimal value of $g_{k_0}(s,t)$ with
respect to $t$).

Now that we introduced this reparametriaztion of the problem, we want to show
that the function $t^*(s)$ is piecewise linear and that in dimension $d=2$ it is
possible to enumerate its pieces and minimize $g_{k_0}$ with a lower complexity
than before.

In dimension 2, the constraints on matrix $M$ leaves us only two choice: either
$\begin{bmatrix}-\frac{1}{\sqrt{2}} & \frac{1}{\sqrt{2}}\end{bmatrix}$ or
$\begin{bmatrix}\frac{1}{\sqrt{2}} & -\frac{1}{\sqrt{2}}\end{bmatrix}$. We can
pick either without loss of generality. We settle for
$\begin{bmatrix}\frac{1}{\sqrt{2}} & -\frac{1}{\sqrt{2}}\end{bmatrix}$. We
assume for the rest of this section that $s_1 \leq  \cdots \leq  s_k$

It is also possible to enumerate the valid $\varphi$ maps in dimension 2. There are
$p+1$ different cells given by: $(-\infty,s_1], [s_1, s_2], \cdots, [s_{p-1},
s_p], [s_p, +\infty)$. And the $\varphi$ map associated with the $l^\text{th}$ cell of
this list is
$$\varphi:k\in p(\mu) \mapsto \begin{cases}
  2 &\text{if } k<l \\
  1&\text{otherwise}
\end{cases}.$$
This is due to the fact that $M(s_k - s) = \frac{1}{\sqrt{2}} (s_k - s)
\begin{bmatrix} 1 & -1 \end{bmatrix}$. So when $s\leq s_k$, $M_2(s_k -s) \leq  M_1(s_k
-s)$ and $M_1(s_k-s)\leq M_2(s_k-s)$ when $s\geq s_k$.

We let $t_{k,\varphi}:s \mapsto M_{\varphi(k)}(s_k -s)) + t_k$. It is a linear function
of $s$.

\begin{lemma}
  The function $t_k(\cdot)$ such that $\forall \varphi,\,\forall x\in S(\varphi),\, t_k(s) =
  t_{k,\varphi}(s)$ is well-defined and piecewise linear of slope $\frac{1}{\sqrt{2}}$
  before $s_k$ and $-\frac{1}{\sqrt{2}}$ after.
\end{lemma}
\begin{proof}
  Since the $S(\varphi)$ provide a tesselation of the space, we only need to check
  that given two maps $\varphi_1$ and $\varphi_2$, the value at the intersection of their
  cell boundaries is equal. The only points at which cells have a non empty
  intersect are the $s_k$'s. These cells are associated with the map $\varphi_1$ that
  maps every index stricly below $k$ to $2$ and the other ones to $1$ and $\varphi_2$
  the one that maps every index below or equal to $k$ to $2$ and the other ones
  to $1$. But $t_{k,\varphi_1}(s_k) = t_k = t_{k,\varphi_2}(s_k)$, hence $t_k(\cdot)$ is
  well defined and on each of the $S(\varphi)$ it is linear thus it is piecewise
  linear. Moreover, using the cells described earlier, we have that for $s<s_k$
  (resp. $s>s_k$), the $\varphi$ map associated with the cell in which $s$ lives maps
  $k$ to $2$ (resp. $1$), thus $t_k(\cdot)$ is linear on $s\leq s_k$ (resp. $s\geq s_k$)
  with slope $\frac{1}{\sqrt{2}}$ (resp. $-\frac{1}{\sqrt{2}}$).
\end{proof}

The Figure \ref{fig:samplet_star} gives an example of a construction of $t^*$
and the $t_k(\cdot)$. The red dashed line represent here each individual
$t_k(\cdot)$ where we can easily see the increasing part up to $s_k$ followed by
a decreasing part.

\begin{lemma}
  $t^*_\varphi$ and $t^*$ are well-defined and piecewise linear.
\end{lemma}
\begin{proof}
$g_{k_0,\varphi}$ is differentiable hence, $t^*_\varphi$ is such that $\frac{\partial
g_{k_0,\varphi}}{\partial t}(s, t^*_\varphi(s)) = 0$. The partial derivative of $g_{k_0,\varphi}$
with respect to its second argument is
$$ \frac{\partial g_{k_0,\varphi}}{\partial t}(s,t) = dw_0 (t-t_0) + \sum_{k\in p(\mu)}
  w_k (t - t_{k,\varphi}(s))_- $$

Since $g_{k_0,\varphi}(s,\cdot)$ is strongly convex it admits exactly one minimizer,
thus, there exists $t_\varphi^*(s)$ such that $t_\varphi^*(s)$ is the minimizer of
$g_{k_0,\varphi}$. However, the negative part in the sum makes it so we cannot solve
easily for $\frac{\partial g_{k_0,\varphi}}{\partial t}(s,t^*_\varphi(s)) = 0$. But, since
$g_{k_0, \varphi}$ is convex, $t^*_\varphi$ is continuous and then $K(s) \triangleq
\mathds{1}\{t^*_\varphi(\cdot) \leq  t_{k,\varphi}(\cdot)\}$ is piecewise constant and only
jumps when $t^*_\varphi$ meets one of the $t_{k,\varphi}$. Thus, if we know $K(s)$ and
$t^*_\varphi(s)$ is different from all the $t_{k,\varphi}(s)$ we can differentiate $t^*_\varphi$
in $s$ giving us the following expression:
$$ \frac{\partial t^*_\varphi}{\partial s}(s) =
  -\frac{\sum_{k\in K(s)} w_k M_{\varphi(k)}}{dw_0 + \sum_{k\in K(s)} w_k} $$
which is constant while $K(s)$ stays the same, meaning that $t^*_\varphi(s)$ is
actually piecewise linear. Also its slope is always in
$(-\frac1{\sqrt2},\frac1{\sqrt2})$.

$t^*$ is equal to $t^*_\varphi$ on $S(\varphi)$ meaning that $t^*$ is also piecewise linear
and witth slopes in $(-\frac1{\sqrt2}, \frac1{\sqrt2})$. This means that $t^*$
will only meet $t_k(\cdot)$ twice, once in its ascending part and once in its
descending part at which point its slope changes. The only other points at which
the slope changes are the boundaries of cell, meaning at the $s_k$'s.

\begin{itemize}
  \item Let $s\leq s_k$ such that $t^*(s) = t_k(s)$, then
    $K(s^+) = K(s^-) \uplus \{k\}$
  \item Let $s = s_k$,
    $\frac{\partial t^*}{\partial s}(s^+)
      = \frac{\partial t^*}{\partial s}(s^-)
        - \frac2{\sqrt2}\cdot \frac{w_k }{dw_0 + \sum_{k\in K(s)} w_k}$
  \item Let $s\geq s_k$ such that $t^*(s) = t_k(s)$, then
    $K(s^+) = K(s^-) \backslash \{k\}$
\end{itemize}
Moreover, because the $(s_k, t_k)$ are Pareto points, $t^*$ will meet the
ascending part in ascending order, the $s_k$ in ascending order and the
descending part in ascending order.

Also, at $-\infty$, $t^*(s) = 0$ and it has null derivative. Using that and the
rules above, we can track $t^*(s)$ in an efficient way. An example is provided
in Figure~\ref{fig:samplet_star}.
\begin{figure}[ht]
  \centerline{\includegraphics[width=\linewidth]{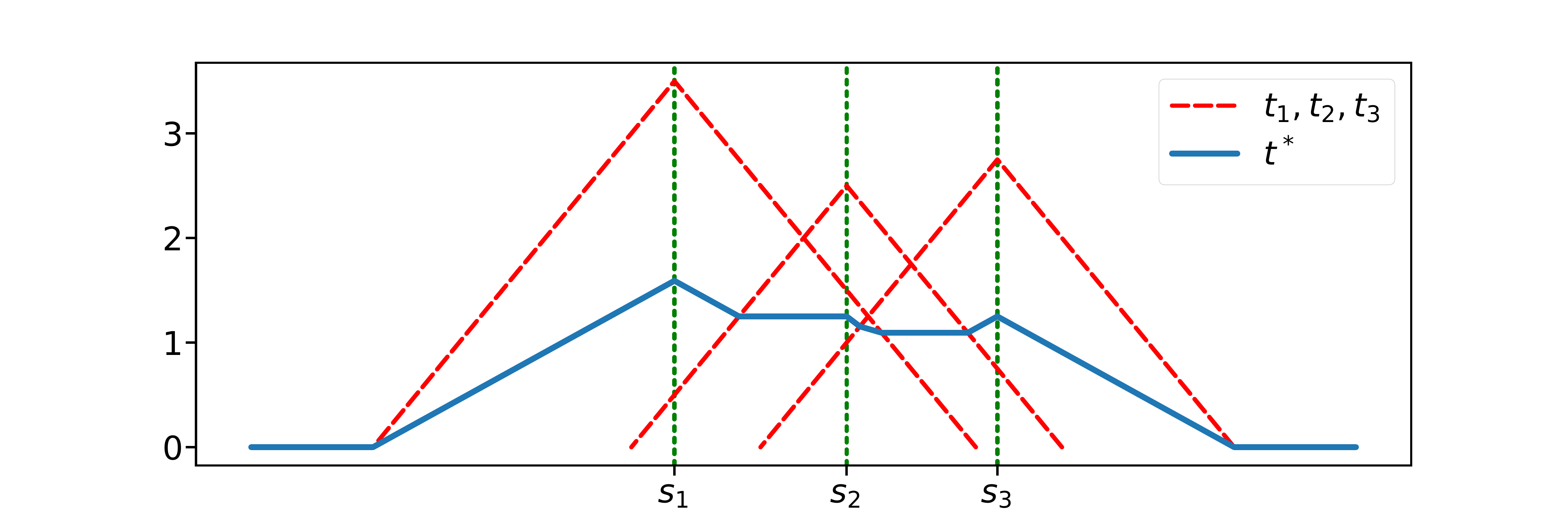}}
  \caption{Example of $t^*$ tracking}
  \label{fig:samplet_star}
\end{figure}
\end{proof}

We can also compute the minimizer and the minimal value of $g_{k_0}(s, t^*(s))$
on a linear part of $t^*(s)$ by differentiating $g_{k_0,\varphi}(s,t)$ in $s$ with
$K(s)$ held constant. This gives rise to an algorithm which has
complexity of order $p\log p$ where you need $p\log p$ operations to sort the
Pareto front and then with complexity $p$ you can range over the $s$-space
keeping track of some sums and of $t^*$ and computing the minimal value of
$g_{k_0}(s,t^*(s))$ along with its minimizer. This reduces the run-time cost of
the transport cost computation for adding a point from $p^2$ to $p\log p$ and
for any point to $K p + p \log p$ as the sorting cost can be dampened.

Moreover, while running Track-and-Stop, only a single point is ever updated
between two subsequent minimal transportation cost computations meaning that the
sort can be updated in linear time.

\subsection{Experiments}
We made some experiments to highlight the gain from using the improved algorithm
for 2d rather than the generic one. For that, we picked $p$ points forming a
Pareto set in a $10×10$ square and we added a unique non Pareto optimal point at
$0$ in our point cloud. Then we measured the time used by each algorithm to
compute the minimum transportation cost against a random vector of weights. We
repeated the operation a thousand times for each tested $p$. We then reported
$t_2$ the average time taken by the improved algorithm for an iteration, $t_n$
the average time taken by the generic algorithm for an iteration and $r$ the
ratio between $t_n$ and $t_2$. These experiments were done on a single core of
an Intel(R) Core(TM) i5-6300U CPU.

\begin{table}[ht]
  \caption{Comparison between the improved and generic algorithm}
  \begin{center}
    \begin{tabular}{|r|c|c|c|} \hline
      $p$ & $t_2$ (s) & $t_n$ (s) & $r$\\\hline
      2 & \num{7.33e-04} & \num{1.01e-03} & $1.38$ \\
      4 & \num{1.05e-03} & \num{1.97e-03} & $1.88$ \\
      8 & \num{1.69e-03} & \num{4.47e-03} & $2.65$ \\
      16 & \num{3.00e-03} & \num{1.25e-02} & $4.17$ \\
      32 & \num{5.68e-03} & \num{3.47e-02} & $6.11$ \\
      64 & \num{1.02e-02} & \num{1.16e-01} & $11.4$ \\
      128 & \num{2.25e-02} & \num{4.48e-01} & $19.9$ \\
      256 & \num{4.91e-02} & \num{1.77e+00} & $36.0$ \\
      512 & \num{8.77e-02} & \num{7.66e+00} & $87.3$ \\ \hline
    \end{tabular}
  \end{center}
\end{table}

\section{Experiments}\label{app:experiment}
In this Appendix, we give additional details about the experiments we ran. The
first round of experiments was based on resampling real life data from the
study~\cite{Munro2021}. We report in Table~\ref{tab:expmean} and
\ref{tab:expvar}, the value from the study that we used to generate the data.
Moreover the last column of Table~\ref{tab:expmean} gives the optimal weight
associated with each arm obtained by solving the optimization problem of
Proposition~\ref{prop:lbd}. Lines written in a bold font correspond to the
Pareto optimal arms and the underlined entry highlights a non-Pareto optimal arm
that needs a lot of samples.

\begin{table}[ht]
  \caption{Means and optimal weights of the different arms}
  \label{tab:expmean}
  \begin{center}
    \begin{tabular}{|c|c|c|c|c|l|}\hline
      Dose 1/Dose 2 & Dose 3 (booster) & Anti-spike IgG & NT$._{50}$ & cellular
        response & \multicolumn{1}{c|}{$w^*$} \\\hline
      \multirow{10}{100pt}{Prime BNT/BNT}& ChAd & 9.5 & 6.86 & 4.56 & 0.0077 \\
      & NVX & 9.29 & 6.64 & 4.04 & 0.0016 \\
      & NVX Half & 9.05 & 6.41 & 3.56 & 0.0007 \\
      & BNT & 10.21 & 7.49 & 4.43 & 0.023 \\
      & BNT Half & 10.05 & 7.2 & 4.36 & 0.0048 \\
      & VLA & 8.34 & 5.67 & 3.51 & 0.00066 \\
      & VLA Half & 8.22 & 5.46 & 3.64 & 0.00079 \\
      & Ad26 & 9.75 & 7.21 & 4.71 & 0.018 \\
      & \textbf{m1273} & \textbf{10.43} & \textbf{7.61} & \textbf{4.72} &
        \textbf{0.14} \\
      & CVn & 8.94 & 6.19 & 3.84 & 0.0011 \\ \hline
      \multirow{10}{100pt}{Prime ChAd/ChAd}& ChAd & 7.81 & 5.26 & 3.97 & 0.0014
      \\
      & NVX & 8.85 & 6.59 & 4.73 & 0.021 \\
      & NVX Half & 8.44 & 6.15 & 4.59 & 0.0089 \\
      & BNT & 9.93 & 7.39 & 4.75 & 0.025 \\
      & BNT Half & 8.71 & 7.2 & 4.91 & \underline{0.35} \\
      & VLA & 7.51 & 5.31 & 3.96 & 0.0014 \\
      & VLA Half & 7.27 & 4.99 & 4.02 & 0.0015 \\
      & Ad26 & 8.62 & 6.33 & 4.66 & 0.013 \\
      & \textbf{m1273} & \textbf{10.35} & \textbf{7.77} & \textbf{5.0} &
        \textbf{0.38} \\
      & CVn & 8.29 & 5.92 & 3.87 & 0.0012 \\ \hline
    \end{tabular}
  \end{center}
\end{table}

\begin{table}[ht]
  \caption{Pooled variance for each immunogenicity trait}
  \label{tab:expvar}
  \begin{center}
    \begin{tabular}{|c|c|c|c|}
      \hline
      & Anti-spike IgG & NT$._{50}$ & cellular response \\ \hline
      Variance & 0.70 & 0.83 & 1.54 \\ \hline
    \end{tabular}
  \end{center}
\end{table}

\end{document}